%% file: neurips_2026.tex
\documentclass{article}

\usepackage[main, final]{neurips_2026}

\usepackage[utf8]{inputenc} 
\usepackage[T1]{fontenc}    
\usepackage{hyperref}       
\usepackage{url}            
\usepackage{booktabs}       
\usepackage{amsthm}         
\usepackage{nicefrac}       
\usepackage{microtype}      
\usepackage{xcolor}         
\usepackage{enumitem}

\usepackage{amsmath}
\usepackage{amssymb}
\usepackage{mathtools}
\usepackage{enumitem}
\usepackage{multirow}
\usepackage{algorithm}
\usepackage{algorithmic}
\usepackage{float}
\usepackage{mdframed}
\usepackage{subcaption}
\usepackage{graphicx}
\usepackage{xspace}
\usepackage{bm}
\usepackage{wrapfig}

\definecolor{frozencol}{RGB}{30,80,200}
\definecolor{traincol}{RGB}{200,40,40}
\newcommand{\fcore}[1]{{\color{frozencol}\bm{#1}}}
\newcommand{\tcore}[1]{{\color{traincol}\underline{\bm{#1}}}}
\newcommand{\fcoretext}[1]{{\color{frozencol}\textbf{#1}}}
\newcommand{\tcoretext}[1]{{\color{traincol}\underline{\textbf{#1}}}}

\theoremstyle{plain}
\newtheorem{theorem}{Theorem}[section]
\newtheorem{proposition}[theorem]{Proposition}
\newtheorem{lemma}[theorem]{Lemma}

\theoremstyle{remark}
\newtheorem{remark}[theorem]{Remark}

\newcommand{\ten}[1]{\bm{\mathcal{#1}}}
\newcommand{\mat}[1]{\mathbf{#1}}
\newcommand{\method}{\textsc{FuRA}\xspace}
\DeclareMathOperator{\rank}{rank}
\DeclareMathOperator{\col}{col}
\DeclareMathOperator{\row}{row}

\newif\ifshowcomments
\showcommentstrue

\title{\method: Full-Rank Parameter-Efficient Fine-Tuning with Spectral Preconditioning}

\author{
Yequan Zhao$^{1}$ \quad
Ruijie Zhang$^{1}$ \quad
Liyan Tan$^{1}$ \quad
Niall Moran$^{2}$ \quad
Tong Qin$^{2}$ \quad
Zheng Zhang$^{1}$ \\
$^{1}$University of California, Santa Barbara \qquad
$^{2}$Amazon Lab126 \\
\texttt{\{yequan\_zhao, ruijiezhang, liyan\_tan\}@ucsb.edu} \\
\texttt{\{nialmora, tqinmath\}@amazon.com} \\
\texttt{zhengzhang@ece.ucsb.edu}
}

\begin{document}

\maketitle

\begin{abstract}
Both full fine-tuning (Full FT) and parameter-efficient methods like LoRA add weight updates without regard to the spectral structure that pretraining has established. This allows noisy gradients from a small fine-tuning distribution to freely perturb the robust features learned through pretraining. We first identify \emph{spectral preconditioning} as the key missing ingredient: reparameterizing each weight $\mat{W}$ through its full-rank SVD and freezing one singular basis confines every update to the pretrained column space, yielding a preconditioned optimizer that outperforms unconstrained Full FT at the same parameter count. To make this insight practical, we propose \method\ (\textbf{Fu}ll-\textbf{R}ank \textbf{A}daptation), which factorizes $\mat{W}$ via a block tensor-train decomposition $\mat{W}=\mat{L}\,\mat{S}\,\mat{R}$: the large core $\mat{L}$ is frozen at the pretrained block-wise SVD basis while only the small core $\mat{R}$ and per-block singular values $\mat{S}$ are trained. This single design choice simultaneously delivers full-rank spectral preconditioning, full-rank update capacity, and parameter, step time, memory efficiency on par with LoRA. 
\method\ outperforms Full FT on LLM fine-tuning ($+1.37$ on LLaMA-3-8B commonsense reasoning), LLM math reinforcement learning, and VLM visual instruction tuning. The 4-bit quantized version QFuRA also outperforms QLoRA. Code is available at \url{https://github.com/yequanzhao/FuRA}.

\end{abstract}

\section{Introduction}
\label{sec:intro}

\begin{wrapfigure}{r}{0.38\linewidth}
	\vspace{-10pt}
	\centering
	\includegraphics[width=\linewidth]{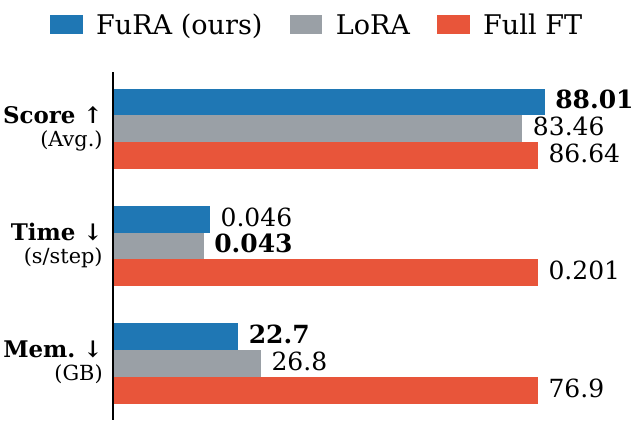}
	\caption{\method\ delivers higher accuracy than Full FT with LoRA-level runtime and the lowest GPU memory.}
	\label{fig:teaser}
	\vspace{-10pt}
\end{wrapfigure}

Large language models~\cite{grattafiori2024llama3,qwen3} (LLM) acquire rich, transferable representations during pretraining on web-scale corpora. Fine-tuning these models on a small, task-specific dataset, whether through supervised finetuning (SFT) or reinforcement learning with verifiable rewards (RLVR) ~\cite{shao2024deepseekmath,yu2025dapo}, can unlock strong downstream performance. Full fine-tuning (Full FT) updates every parameter and remains the performance ceiling on most benchmarks, but incurs prohibitive memory and compute costs. Parameter-efficient methods such as LoRA~\cite{hu2022lora} circumvent this cost by constraining the update $\Delta \mat{W}$ to a low-rank subspace, yet a persistent accuracy gap relative to Full FT ~\citep{schulman2025loranoregret} motivates continued research into more expressive alternatives including~\cite{liu2024dora, meng2024pissa, zhang2023adalora, liu2025lift}.

We argue that both Full FT and LoRA share a deeper limitation: \textbf{neither respects the spectral structure of the pretrained weight.} The update is unconstrained in direction and can push the model along axes orthogonal to the feature manifold that pretraining established. Specifically, the update $\Delta \mat{W}$, whether computed directly from gradients (Full FT) or accumulated in a low-rank adapter $\mat{B}\mat{A}$ (LoRA), is added to $\mat{W}$ without regard to its learned spectral geometry.  Because the fine-tuning data is orders of magnitude smaller and narrower in distribution, its gradients are noisier estimates of the true task-relevant directions. Allowing these noisy updates to freely perturb robust pretrained features risks degrading generalization, a failure mode closely related to catastrophic forgetting \cite{kirkpatrick2017overcoming}.
This observation suggests a different design principle: rather than adding an unconstrained perturbation to $\mat{W}$, fine-tuning should \textbf{re-align the pretrained features} in a way that is guided by, and confined to, the subspace the model has already learned. If the pretrained weight encodes a coordinate system for useful representations, adaptation should re-orient within that coordinate system rather than demolish and rebuild it.

Through a controlled study of SVD FT  (\S\ref{sec:method:motivation}), we further validate that confining weight updates to the pretrained column space can produce superior performance in LLM fine tuning. However, directly doing so can introduce a large number of training variables like full FT. This motivates us to develop \method\ (\textbf{Fu}ll-\textbf{R}ank \textbf{A}daptation, Figure~\ref{fig:fura}), a parameter-efficient method that
simultaneously delivers 1) full-rank update capacity, 2) high accuracy through full-rank spectral preconditioned update, and 3) LoRA-level parameter efficiency.

\method\ outperforms other PEFT methods and Full FT without sacrificing inference efficiency: $+2.87$ over DoRA \cite{liu2024dora} and $+1.37$ over Full FT on LLaMA-3-8B Commonsense SFT, outperforms Full FT on Qwen-1.7B/7B math RLVR, $+1.1$ over Full FT on VLM visual instruction tuning, all with similar memory and wall-clock step time of LoRA. The 4-bit quantized version QFuRA also surpasses QLoRA \cite{dettmers2023qlora} by $+2.51$ on LLaMA-3-70B math fine-tuning.

\textbf{Contributions.} Our research contributions are summarized below.
\begin{itemize}[leftmargin=*,itemsep=1pt,topsep=1pt]
\item We study the spectral properties of LLM fine-tuning and identify spectral preconditioning as a key missing ingredient. We show that both full fine-tuning and LoRA ignore the pretrained spectral structure, while controlled SVD-based fine-tuning demonstrates that aligning updates with this structure yields better performance than full fine-tuning.

\item Based on these observations, we propose \method, a block tensor-train factorization that makes spectral preconditioning a practical PEFT method. A single architectural choice simultaneously achieves full-rank spectral preconditioning, full-rank update capacity, and LoRA-level parameter and computing efficiency.

\item We prove two key properties of FuRA: although it trains only a small fraction of parameters, it can realize full-rank updates; and its effective update acts as a column-space projection combined with singular-value preconditioning.

\item We demonstrate that \method\ outperforms Full FT on both SFT and RLVR tasks with only $<2\%$ trainable parameters and no task-specific rank tuning, establishing a new state of the art for parameter-efficient LLM adaptation. 

\end{itemize}

\begin{figure}[t]
  \centering
  \includegraphics[width=0.8\linewidth]{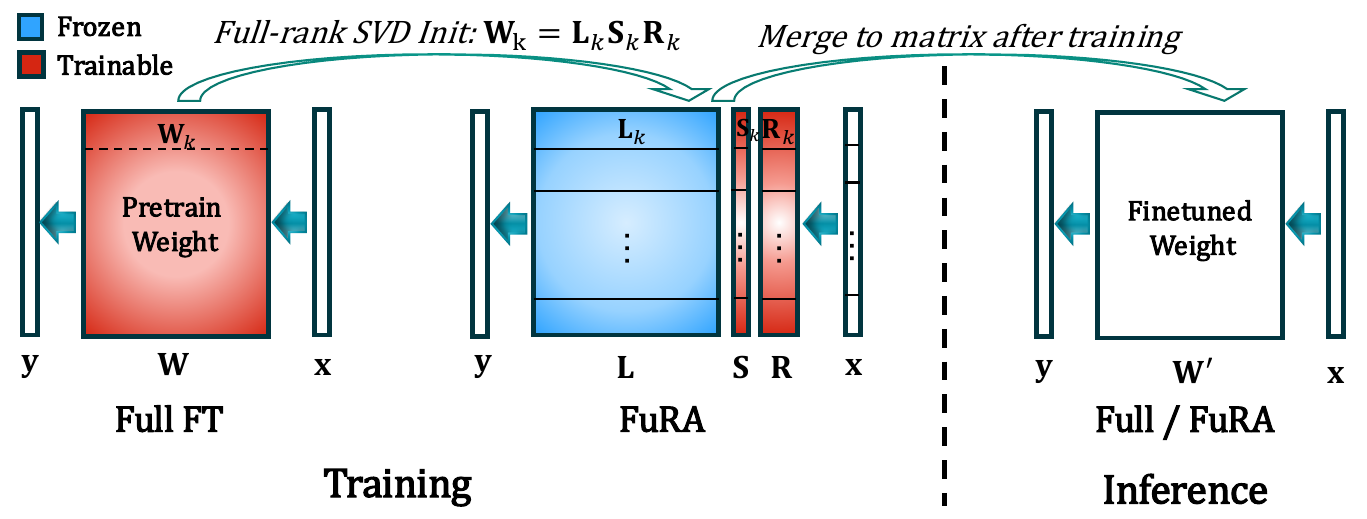}
  \caption{\method\ replaces each pretrained linear layer $\mat{W}$ with a lossless block tensor-train factorization $\mat{W} = \mat{L}\,\mat{S}\,\mat{R}$. Tensors are flattened to matrix for better demonstration. Each slice is initialized by the full-rank SVD of weight block $\mat{W}_k$, and performs spectral preconditioned update.}
  \label{fig:fura}
  \vspace{-5pt}
\end{figure}

\section{Background and Related Work}
\label{sec:related}

\textbf{PEFT methods for LLMs.} LoRA \cite{hu2022lora} constrains the weight update to $\Delta \mat{W} = \mat{B}\mat{A}^\top$ with rank $r \ll \min(d_\text{in}, d_\text{out})$, training only $\mathcal{O}(r(d_\text{in}+d_\text{out}))$ parameters per layer. This low-rank constraint primarily serves parameter efficiency rather than reflecting task-specific inductive bias. Subsequent work improves LoRA along several directions: QLoRA \cite{dettmers2023qlora} adds weight quantization, DoRA \cite{liu2024dora} decouples magnitude and direction, AdaLoRA \cite{zhang2023adalora} adapts rank across layers, and VeRA \cite{kopiczko2024vera} shares fixed random bases while training only diagonal scales. 
Tensor-train decomposition \cite{oseledets2011tt,novikov2015ttnet,qiu2024btt} offer another route; adapters such as LoRETTA \cite{yang2024loretta} and LoRTA \cite{hu2024lorta} train all tensor cores but leave $\Delta \mat{W}$ low-rank.
Despite these advances, a performance gap with Full FT remains.

\textbf{PEFT with higher or full rank capacity.} A natural remedy is to lift the rank cap. MoRA \cite{jiang2024mora} reshapes the adapter into a square matrix to allow higher-rank updates; RandLoRA \cite{albert2025randlora} uses fixed random bases with trainable scales; BOFT \cite{liu2024boft} applies a butterfly-factored orthogonal transform, and 
QuanTA \cite{chen2024quanta}, FourierFT \cite{gao2024fourierft}, etc. $\text{S}^2$FT \cite{yang2024s2ft} and LIFT \cite{liu2025lift} train structured- or unstructured-sparse subsets of weights, achieving full-rank updates but requiring gather/scatter operations that are less tensor-parallel friendly. These methods achieve nominally full-rank $\Delta \mat{W}$ at the LoRA budget, but have not consistently closed the gap to Full FT on standard benchmarks. We argue that rank alone is insufficient; the \emph{subspace} in which the update operates is equally important.

\textbf{Spectrum-informed adapters.} A related line of work leverages the pretrained weight’s SVD to guide adaptation. PiSSA \cite{meng2024pissa} initializes LoRA from top-$r$ singular directions, MiLoRA \cite{wang2024milora} uses bottom-$r$, Spectral Adapter \cite{zhang2024spectral} parameterizes updates in the spectral basis, and SVFT \cite{lingam2024svft} trains a dense middle matrix between frozen SVD factors. While these approaches incorporate spectral information, they face two limitations. First, the low-rank constraint forces a choice of spectral slice, yet the optimal subspace is task-dependent: SFT emphasizes principal directions, whereas RLVR favors off-principal ones \cite{zhu2025rlvrpath}. Second, SVD-based initialization is only a soft constraint, allowing updates to drift away from the pretrained subspace during training \cite{yin2025evaluating}. 

\section{Spectral Preconditioning Improves LLM Fine-tuning}
\label{sec:spectrum}

Throughout this paper, $\mat{W} \in \mathbb{R}^{d_\text{out} \times d_\text{in}}$ denotes a pretrained linear layer with forward pass $\mat{y} = \mat{W}\mat{x}$. And we use $\mat{W}'$ to denote the fine-tuned weight. 

\subsection{Observations from Full FT Training Dynamics}

Both Full FT and existing PEFT methods update weights without considering the spectral structure learned during pretraining, risking disruption of the pretrained feature geometry. We hypothesize that fine-tuning should instead preserve and exploit this structure. To investigate this, we analyze the training dynamics of Full FT on LLaMA-3-8B fine-tuned on Math-10K \cite{hu2023llmadapters} (Figure~\ref{fig:motivation_svdft}).


\paragraph{Gradients are spectrally diffuse, while weight spectrum remain stable.}
Let $\mat{W} = \mat{U}\mat{\Sigma}\mat{V}^\top$ denote the SVD of a pretrained weight. We quantify how much a matrix $\mat{M}$ lies in the pretrained column space $\col(\mat{U})$ using the \emph{column-space energy ratio}:
\begin{equation}
  \rho(\mat{M};\,\mat{U})
  \;=\;
  \frac{\|\mat{U}^\top\mat{M}\|_F^2}{\|\mat{M}\|_F^2}\,.
  \label{eq:col_energy_ratio}
\end{equation}
 We have $\rho\in [0, 1]$ and $\rho=1$ indicates full alignment with $\mathrm{span}(\mat{U})$. Gaussian random matrices has expected energy ratio $\rank(\mat{U})/d_\text{out}$ (see Appendix \ref{sec:appendix:proofs}). Tracking the energy ratio of the gradient $\mat{G}_{\mat{W}}$ during Full FT (Figure~\ref{fig:motivation_svdft}a), we observe that $\rho(\mat{G}_{\mat{W}},\mat{U})$ stays near the random baseline, showing gradients are updating all directions and are not particularly aligned with pretrained directions. In contrast, $\rho(\mat{W}',\mat{U})\approx 1$ throughout training, indicating that the Full FT implicitly preserves the weight’s spectral structure despite diffusive gradient updates. 

\paragraph{Singular values shift selectively; singular vectors rotate broadly.}
Comparing the SVDs of $\mat{W}'$ and $\mat{W}$ (Figure~\ref{fig:motivation_svdft}b), we find that only a few singular values change significantly, while most remain close to their pretrained values. 
In contrast, the singular vectors exhibit substantial rotation: cosine similarities between columns of $\mat{U}'$ and $\mat{U}$ vary widely, with no clear bias toward principal or off-principal directions. 
The above findings suggest that Full FT, despite being unconstrained, implicitly preserves the pretrained column space, concentrating spectral changes on a few singular values while allowing broad directional rotation. This raises a natural question: {\it what if we explicitly enforce these patterns as architectural constraints?} 

\begin{figure}[t]
  \centering
  \includegraphics[width=\linewidth]{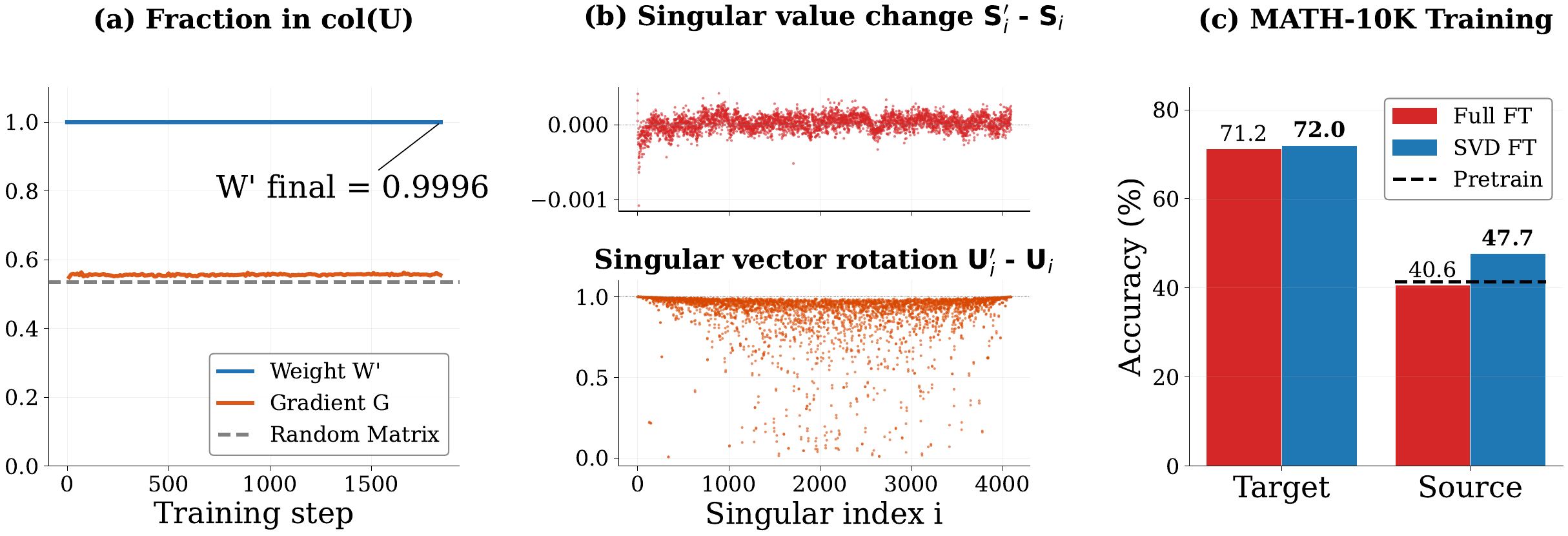}
  \caption{\textbf{(a)}~Gradient lives outside singular basis of pretrained weight, but weight keep stable. \textbf{(b)}~Singular values shift selectively, while singular vectors rotate broadly. (a) (b) demonstrates layer 15 q\_proj result, the full sweeps are in Figures~\ref{fig:motivation_svdft_a_sweep} and~\ref{fig:motivation_svdft_b_sweep} (Appendix). \textbf{(c)}~SVD FT outperforms Full FT on both target domain and source domain. The same analysis for Qwen3-1.7B math RLVR is in Figure~\ref{fig:motivation_svdft_rl} (Appendix).
  }
  \label{fig:motivation_svdft}
\end{figure}

\subsection{Insights from SVD Fine-Tuning Experiment}
\label{sec:method:motivation}
We study a controlled SVD Fine-tuning (SVD FT) to validate the above hypothesis. We decompose each pretrained weight via SVD, $\mat{W} = \mat{U}\mat{\Sigma}\mat{V}^\top$, freeze the left-singular basis $\mat{U}$, and fine-tune only $\mat{\Sigma}$ and $\mat{V}^\top$. This SVD FT has identical parameter count and initialization to full FT, differing only in parameterization. It induces a \emph{spectral preconditioning}, where the effective update from $\mat{V}^\top$ is
\begin{equation}
  \Delta\mat{W}\big|_{\mat{V}}
  \;=\; -\eta\;\mat{U}\mat{\Sigma}^2\mat{U}^\top\;\mat{G}_{\mat{W}}\,,
  \label{eq:svdft_preconditioner}
\end{equation}
with $\mat{G}_{\mat{W}} = \partial\mathcal{L}/\partial\mat{W}$. The projection $\mat{U}\mat{U}^\top$ restricts updates to $\col(\mat{U})$, formalizing the column-space preservation observed earlier, while $\mat{\Sigma}$ separates magnitude from direction, aligning with the distinct spectral dynamics.

We evaluate LLaMA-3-8B trained with Full FT and SVD FT on Math-10K, measuring both target domain math reasoning (GSM8K \cite{cobbe2021gsm8k}) and source domain commonsense reasoning performance (Figure~\ref{fig:motivation_svdft}c). We observe that: (1) SVD FT improves target-domain performance, indicating that gradients within $\col(\mat{U})$ suffice for effective learning; (2)  SVD FT improves source-domain generalization, whereas Full FT degrades relative to the pretrained model. 

This experiment suggests that {\bf spectral preconditioning better preserves pretrained structure}, mitigating forgetting and promoting generalizable representations.

\section{The \method\ Framework}
\label{sec:method}

SVD FT demonstrates the benefits of spectral preconditioning but is impractical: it retains a full-size trainable factor and doubles forward-pass cost. 
We propose \method, which addresses both issues using a Block Tensor-Train (BTT) parameterization.

\method reparameterizes each linear layer with a 2-core Block Tensor-Train (BTT) decomposition ~\cite{qiu2024btt}, which can be interpreted as block-wise SVD. The full-rank 2-core BTT can be used to represent arbitrary matrix. By carefully selecting a special block size, we can perform spectral preconditioned update as for the SVD FT, while with a cost on par with LoRA. 

\paragraph{Block tensor-train (BTT) decomposition.}
We focus on 2-core decomposition case.
Given $\mat{W}\in\mathbb{R}^{d_\text{out}\times d_\text{in}}$, choose an $m\times n$ grid such that $d_\text{out}=m a$ and $d_\text{in}=n b$. Each block $\mat{W}_{ij}\in\mathbb{R}^{a\times b}$ is factorized into $\mat{L}_{ij}\in\mathbb{R}^{a\times r}$ and $\mat{R}_{ij}\in\mathbb{R}^{r\times b}$, forming cores $\ten{L}\in\mathbb{R}^{m\times n\times a\times r}$ and $\ten{R}\in\mathbb{R}^{m\times n\times r\times b}$. The matrix-vector product $\mat{y}=\mat{W}\mat{x}$ becomes
\begin{equation}
  y_{\alpha\beta}
  = \sum_{\gamma\sigma} \ten{L}_{\alpha\beta\gamma\sigma}
    \sum_{\delta} \ten{R}_{\sigma\beta\gamma\delta}\, \mat{x}_{\gamma\delta}.
  \label{eq:btt}
\end{equation}
At full rank $r=\min(a,b)$, this BTT is applying full-rank SVD for each individual block and hence the representation is equivalent to the original dense matrix.
For a square, full-rank decomposition
($d_\text{out}=d_\text{in}=d$, $m=n=a=b=r=\sqrt{d}$), the total parameter
count and MVM FLOPs are both $2d^2$~\cite{qiu2024btt}, same as SVD.

\paragraph{The \method\ parameterization.}
We decompose each pretrained weight using a parameter-efficient BTT configuration and keep the diagonal singular values as separate trainable parameters. 
Setting $m=1$ yields blocks with a \emph{large} core $\mat{L}_k\in\mathbb{R}^{d_\text{out}\times r}$ and a \emph{small} core $\mat{R}_k\in\mathbb{R}^{r\times b}$ with $b\ll d_\text{in}$. 
At full rank $r=b$, this decomposition is lossless while only introduces $d_\text{in}b$ additional parameters.
For each block, full-rank SVD yields
$\mat{W}_k=\mat{U}_k\mat{\Sigma}_k\mat{V}_k^\top$, and we set
$\mat{L}_k=\mat{U}_k\!\in\!\mathbb{R}^{d_\text{out}\times b}$,
$\mat{S}_k=\mathrm{diag}(\mat{\Sigma}_k)\!\in\!\mathbb{R}^{b}$,
$\mat{R}_k=\mat{V}_k^\top\!\in\!\mathbb{R}^{b\times b}$.
Stacking blocks, we got BTT cores $\mat{L}=[\mat{L}_1,\dots,\mat{L}_n]$,
$\mat{S}=[\mat{S}_1,\dots,\mat{S}_n]$,
$\mat{R}=[\mat{R}_1,\dots,\mat{R}_n]$.
\method\ \textbf{freezes $\mat{L}$} and trains only $\mat{S}$ and $\mat{R}$. The fine-tuned weight is assembled column-block-wise:
\begin{equation}
  \mat{W}'_k \;=\; \mat{L}_k\,\mathrm{diag}(\mat{S}'_k)\,\mat{R}'_k,
  \qquad
  \mat{W}' \;=\; \bigl[\,\mat{W}'_1 \,\big|\, \mat{W}'_2 \,\big|\, \cdots \,\big|\, \mat{W}'_n\,\bigr],
  \label{eq:fura}
\end{equation}
where $[\cdot|\cdot]$ denotes horizontal (column-block) concatenation. Algorithm~\ref{alg:fura} in Appendix summarizes the procedure. 

\paragraph{Complexity analysis.} The block-SVD initialization is a \textbf{one-time} cost (e.g. $\sim 1$ minute for 8B model). During training, \method\ matches LoRA in both memory and step time (Table~\ref{tab:cost}). The trainable count is $|\mat{R}|+|\mat{S}| = nrb + nr \approx d^{3/2}$, roughly $2\%$ of pretrained weight size at $d=4096$; the $\mat{L}$-stage reduces to a single dense MVM and the $\mat{R}$-stage to a batched BMM, both well supported in deep learning frameworks; after training, the cores merge back into a single dense matrix so there is no additional inference overhead.  We defer the full breakdown to Appendix~\ref{sec:appendix:complexity}.

\begin{figure}[t]
  \centering
  \includegraphics[width=\linewidth]{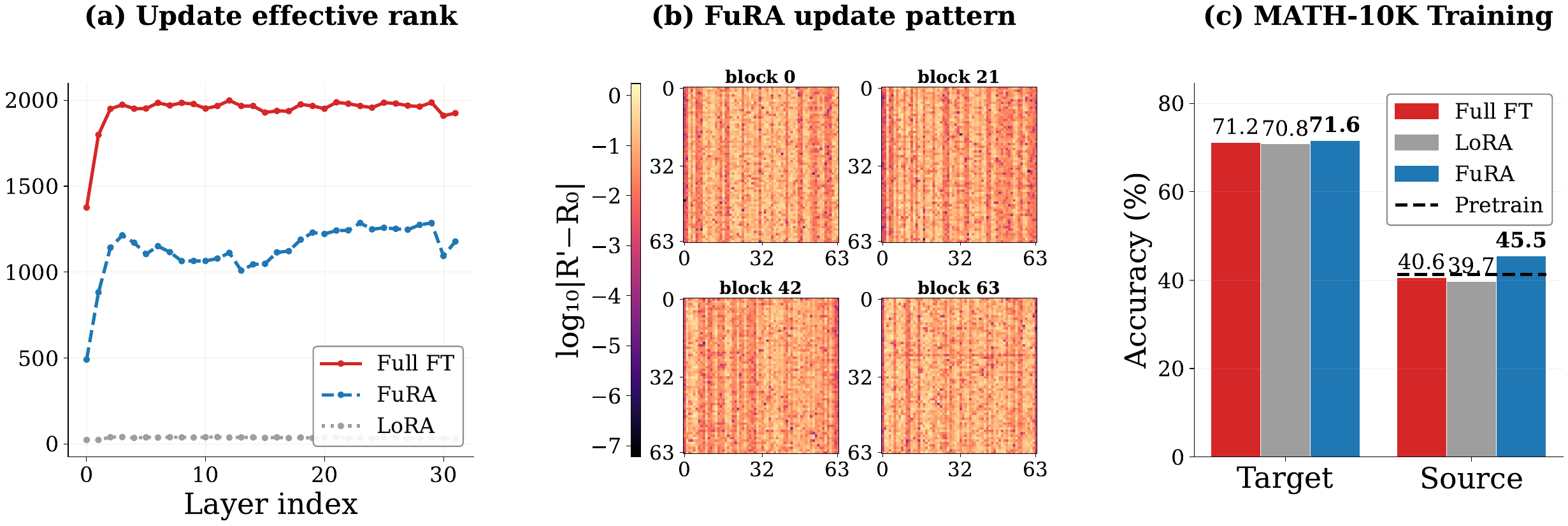}
  \caption{\textbf{(a)}~Effective rank of $\Delta \mat{W}$; full sweep is in Figure~\ref{fig:motivation_fura_a_sweep} (Appendix). \textbf{(b)}~Singular-direction heatmaps show sharper active/inactive separation, reflecting spectral preconditioning from frozen $\mat{L}$. \textbf{(c)}~\method\ better preserves pretrained capabilities. }
  \label{fig:motivation_fura}
  \vspace{-5pt}
\end{figure}

\textbf{Advantages of \method.} This single design choice delivers the following properties
\begin{itemize}[leftmargin=*]
    \item \emph{\textbf{Full-rank capacity.}}  Each block’s update lies in $\col(\mat{L}_k)$, and across blocks these subspaces span the full output space, leaving $\Delta\mat{W}$ unconstrained in rank. The optimizer, rather than a preset rank, controls layer capacity.
    As shown in Fig.~\ref{fig:motivation_fura} (a), the effective rank of \method\ update is significantly higher than LoRA. The lower effective rank compared to Full FT suggests that the spectral preconditioner focuses capacity on task-relevant directions rather than distributing it uniformly.
    Importantly, the full-rank block-wise singular basis is preserved, avoiding commitment to specific spectral subspaces. We formalize this in Section~\ref{sec:theory}, Claim~1. 

\item \emph{\textbf{Spectral preconditioning.}} The spectral preconditioning from SVD FT (Section~\ref{sec:method:motivation}, Eq.~\eqref{eq:svdft_preconditioner}) extends naturally to \method's block structure. The trainable $\mat{S}$ decouples magnitude and direction, enabling finer spectral control, similar in spirit to DoRA~\cite{liu2024dora}. As shown in Figure~\ref{fig:motivation_fura} (b), the optimizer selectively updates important singular directions while preserving others. The full preconditioner is derived in Section~\ref{sec:theory}, Claim~2.

\item \emph{\textbf{Better generalization.}} \method\ outperforms Full FT and LoRA on target-domain math reasoning using only ${\sim}2\%$ trainable parameters, while better preserving (improving) source-domain commonsense reasoning accuracy [Fig.~\ref{fig:motivation_fura} (c)]. By restricting updates to the pretrained column space, it favors reweighting existing features over introducing arbitrary directions that may overfit to the fine-tuning distribution.
\end{itemize}

\begin{remark}[Design space]
The \method\ design has two axes: \textbf{(1)} preserve the output subspace ($m=1$, the default) or input subspace ($n=1$); and \textbf{(2)} keep $\mat{S}$ as a separate trainable vector (the default) or merged
into the trainable / frozen core. We characterize each choice's effective
preconditioner in Section~\ref{sec:theory} (Claim~2, Table~\ref{tab:preconditioners})
and ablate them in Section~\ref{sec:exp:ablation}.
\end{remark}

\begin{table}[t]
\centering
\caption{System overhead comparison. Avg is the average accuracy on LLaMA-3-8B commonsense SFT Detail setups and analysis are in Appendix \ref{sec:appendix:syscost}. }
\label{tab:cost}
\small
\setlength{\tabcolsep}{4pt}
\begin{tabular}{lcccccc}
\toprule
Method & Train.\ (\%) & Extra (\%) & Full-rank $\Delta\mat{W}$? & Step (s) & Memory (GB) & Avg (\%) \\
\midrule
Full FT                 & $100.00$ & $0$    & Yes                  & $0.201$ & $76.9$          & $86.64$ \\
LoRA ($r{=}64$)         & $1.39$   & $1.39$ & No                   & $\textbf{0.043}$ & $26.8$          & $83.46$ \\
DoRA ($r{=}64$)         & $1.42$   & $1.42$ & No                   & $0.045$ & $26.8$          & $85.04$ \\
RandLoRA ($r{=}64$)     & $0.33$   & $2.57$ & No (random)          & $0.059$ & $25.7$          & $83.28$ \\
LIFT ($r{=}32$)         & $6.23$   & $0$    & Yes (sparse)         & $0.144$ & $56.0$          & $87.88$ \\
\textbf{\method\ (ours)}& $1.46$   & $1.46$ & \textbf{Yes (dense)} & $0.046$ & $\mathbf{22.7}$ & $\mathbf{88.01}$ \\
\bottomrule
\end{tabular}
\vspace{-10pt}
\end{table}

\section{Theoretical Insights}
\label{sec:theory}

In this section, we present two theoretical claims regarding the theoretical properties of \method. We support each claim with a short derivation and provide the full proofs to Appendix~\ref{sec:appendix:proofs}.

\paragraph{Claim 1 (Full-rank update capacity).} Although \method\ trains only $\mat{R}$ and $\mat{S}$, the update $\Delta\mat{W}$ has full-rank capacity: it can reach any rank up to $\min(d_\text{in}, d_\text{out})$ whenever the pretrained weight $\mat{W}$ is full rank, in contrast to LoRA-style adapters that hard-cap $\rank(\Delta\mat{W})\le r$.

\begin{proposition}[No explicit rank cap]
\label{prop:fullrank}
Under the \method\ parameterization in Eq.~\eqref{eq:fura}, with $\mat{L}$ frozen and $\mat{S},\mat{R}$ trainable, the update $\Delta\mat{W} = \mat{W}' - \mat{W}$ is not rank-constrained (Lemma~\ref{lem:rank-formal}).
\end{proposition}

\textit{Sketch of Proof.} Per block, $\Delta\mat{W}_k$ lies in $\col(\mat{L}_k)$ (Propositions~\ref{prop:block-constraint-col} and~\ref{prop:block-constraint-row}, appendix); across the $n$ blocks, these subspaces collectively span $\col(\mat{W})$.

\begin{remark}
    While Claim~1 establishes full-rank reachability, the rank measure alone does not capture the geometry of the reachable updates. Two further structural properties follow from the per-block construction (Appendix~\ref{sec:appendix:proofs}): (i) the default $m{=}1$ orientation processes the input as $n$ disjoint block slices, with no cross-block information transfer at the BTT layer; and (ii) each per-block update is confined to $\col(\mat{L}_k)$, so $\Delta\mat{W}$ lies in a product of pretrained subspaces.
    The cross-block locality (i) is the primary expressivity cost, and could be relaxed at extra parameters by prepending an $n{\times}n$ input-mixing matrix before the BTT contraction; we leave this extension to future work.
\end{remark}

\paragraph{Claim 2 (Spectral preconditioning).} The effective update under \method\ is not standard SGD on $\mat{W}$: it applies \textit{column-space projection} and \textit{singular-value preconditioning}, both induced by the frozen block-wise SVD basis. We provide the explanations below.

Consider a single block with SVD $\mat{W}_k = \mat{U}_k \mat{\Sigma}_k \mat{V}_k^\top$ and the $k$-th block input $\mat{x}_k$. Let $\mat{g} = \partial \mathcal{L}/\partial \mat{y}$ and $\mat{G}_k = \mat{g}\mat{x}_k^\top$ denote the unconstrained gradient. Under the default \method\ setting (train $\mat{R}$ and $\mat{S}$, freeze $\mat{L}=\mat{U}$), the forward pass is $\mat{y}_k = \mat{U}_k\,\mathrm{diag}(\mat{S}_k)\,\mat{R}_k\,\mat{x}_k$. A single gradient step yields
\begin{equation}
\Delta \mat{W}_k \;=\; \underbrace{-\eta\, \mat{U}_k\, \mathrm{diag}(\mat{S}_k)^2\, \mat{U}_k^\top\, \mat{G}_k}_{\text{from } \mat{R} \text{ update}} \;+\; \underbrace{\mat{U}_k\,\mathrm{diag}(\Delta \mat{S}_k)\,\mat{V}_k^\top}_{\text{from } \mat{S} \text{ update}},
\label{eq:default-precond}
\end{equation}
where $\mathrm{diag}(\mat{S}_k)^2$ arises because $\mathrm{diag}(\mat{S}_k)$ appears twice: once in the chain rule $\Delta\mat{R}_k = \operatorname{diag}(\mat{S}_k)\, \mat{U}_k^\top \mat{G}_k$
and another is structural in the reconstruction  $\mat{U}_k\,\mathrm{diag}(\mat{S}_k)\,\Delta\mat{R}_k$. The preconditioner $\mat{U}_k\,\mathrm{diag}(\mat{S}_k)^2\,\mat{U}_k^\top$ in the first term combines two effects:
\begin{itemize}[leftmargin=*]
    \item {\bf Column-space projection.} Since $\mat{U}_k$ has orthonormal columns ($\mat{U}_k^\top \mat{U}_k = \mat{I}$, but $\mat{U}_k \mat{U}_k^\top \neq \mat{I}$ for rectangular blocks), the update $\Delta \mat{W}_k$ is restricted to $\mathrm{col}(\mat{U}_k)$, the pretrained left-singular subspace. In \method, $\mat{U}_k$ is always rectangular, making this projection non-trivial.

    \item \textbf{Singular-value preconditioning.} Within $\mathrm{col}(\mat{U}_k)$, gradients are scaled by $\sigma^2$ along each singular direction, amplifying high-signal components and suppressing low-signal ones. This induces an adaptive, spectrum-dependent learning rate. Note that with optimizer like AdamW or Muon that update with normalized/orthogonalized gradient, the update still keeps a $\sigma^p, 1<p<2$ spectral weighting as the structural part $\mat{U}_k\,\mathrm{diag}(\mat{S}_k)\,\Delta\mat{R}_k$ applies outside the optimizer.
\end{itemize}

\paragraph{Design choices.} \method\ exposes two design axes; Table~\ref{tab:preconditioners} enumerates all corners.

\textbf{(1) Subspace preservation.} Choosing orientation $m{=}1$ with $\mat{L}$ frozen (the default) preserves the per-block pretrained \emph{output} subspace: every block update satisfies $\col(\Delta\mat{W}_k) \subseteq \col(\mat{U}_k)$. The symmetric choice $n{=}1$ with $\mat{R}$ frozen instead preserves the per-block pretrained \emph{input} subspace, $\row(\Delta\mat{W}_k) \subseteq \row(\mat{V}_k^\top)$. Both orientations share the same trainable parameter count and full-rank update capacity (Claim~1), but bias the optimizer toward different pretrained directions.

\textbf{(2) Placement of $\mat{S}$.} Keeping $\mat{S}$ as a separate trainable vector (default) yields the $\mathrm{diag}(\mat{S}_k)^2$ preconditioner of Eq.~\eqref{eq:default-precond} plus the magnitude term $\mat{U}_k\,\mathrm{diag}(\Delta\mat{S}_k)\,\mat{V}_k^\top$. Merging $\mat{S}$ into the trainable core removes spectral reweighting (``fair'' projector $\mat{U}_k\mat{U}_k^\top$, MiLoRA-like); merging it into the frozen core produces a principal-biased $\mathrm{diag}(\mat{S}_k)^2$-weighted projector (PiSSA-like). Derivations for all four placements: Appendix~\ref{sec:appendix:precond}.

\begin{table}[t]
\centering
\footnotesize
\setlength{\tabcolsep}{3pt}
\caption{Effective gradient-induced $\Delta \mat{W}_k$ per design corner of \method\ (up to the learning rate $-\eta$). Principal-biased (PiSSA-like) and off-principal-biased (MiLoRA-like) updates arise as two corners of the same framework. The \textbf{default} configuration used in the main experiments is the last row.}
\label{tab:preconditioners}
\begin{tabular}{llc}
\toprule
Trainable side & $\mat{S}$ placement & Effective $\Delta \mat{W}_k$ \\
\midrule
Train $\mat{L}$ (output), freeze $\mat{R}=\mat{V}_k^\top$ & $\mat{S}$ in trainable $\mat{L}$ & $\mat{g}\mat{x}_k^\top\, \mat{V}_k \mat{V}_k^\top$ \\
Train $\mat{L}$ (output), freeze $\mat{R}=\mathrm{diag}(\mat{S}_k)\mat{V}_k^\top$ & $\mat{S}$ in frozen $\mat{R}$ & $\mat{g}\mat{x}_k^\top\, \mat{V}_k \,\mathrm{diag}(\mat{S}_k)^2\, \mat{V}_k^\top$ \\
Train $\mat{R}$ (input), freeze $\mat{L}=\mat{U}_k$ & $\mat{S}$ in trainable $\mat{R}$ & $\mat{U}_k\mat{U}_k^\top\, \mat{g}\mat{x}_k^\top$ \\
Train $\mat{R}$ (input), freeze $\mat{L}=\mat{U}_k\,\mathrm{diag}(\mat{S}_k)$ & $\mat{S}$ in frozen $\mat{L}$ & $\mat{U}_k \,\mathrm{diag}(\mat{S}_k)^2\, \mat{U}_k^\top\, \mat{g}\mat{x}_k^\top$ \\
\textbf{Train $\mat{R}$, keep $\mat{S}$ separate trainable (default)} & separate, trainable & $\mat{U}_k\,\mathrm{diag}(\mat{S}_k)^2\,\mat{U}_k^\top\, \mat{g}\mat{x}_k^\top + \mat{U}_k\,\mathrm{diag}(\Delta \mat{S}_k)\,\mat{V}_k^\top$ \\
\bottomrule
\end{tabular}
\end{table}


\section{Experiments}
\label{sec:exp}


\paragraph{Setup.} \textbf{Tasks.} We evaluate \method\ on two settings of major interest to the LLM community: 1) \emph{Commonsense reasoning SFT}, in which we fine-tune on Commonsense-170K \cite{hu2023llmadapters} and evaluate on the standard eight-task suite (BoolQ, PIQA, SIQA, HellaSwag, WinoGrande, ARC-e, ARC-c, OBQA) \cite{clark2019boolq,bisk2020piqa,sap2019siqa,zellers2019hellaswag,sakaguchi2020winogrande,clark2018arc,mihaylov2018openbookqa}; 2) \emph{Math reinforcement learning with verifiable rewards}, in which we run GRPO \cite{shao2024deepseekmath} on a math prompt set and evaluate on four held-out benchmarks: MATH-500, AMC23, AIME-24, and AIME-25 \cite{hendrycks2021math, amc23, aime24, aime25}. \textbf{Models.} For commonsense SFT we use LLaMA-2-7B \cite{touvron2023llama2} and LLaMA-3-8B \cite{grattafiori2024llama3}; for math RLVR we use Qwen3-1.7B \cite{qwen3} and Qwen2.5-7B \cite{qwen2025qwen25}. We additionally evaluate 4-bit quantized fine-tuning on LLaMA-3-8B and LLaMA-3-70B. \textbf{Baseline methods.} We compare against Full FT and four families of parameter-efficient methods covering the leading LoRA variants: (i)~\emph{low-rank adapters} LoRA \cite{hu2022lora} and DoRA \cite{liu2024dora}; (ii)~\emph{SVD-initialized adapters} PiSSA \cite{meng2024pissa} and MiLoRA \cite{wang2024milora}; and (iii)~\emph{full-rank capacity} RandLoRA \cite{albert2025randlora};. For all baselines we use the rank and other hyperparameters reported as best in the original references. Full hyperparameters and sweeps are in Appendix~\ref{sec:appendix:exp}.


\subsection{Commonsense reasoning SFT (LLaMA-2-7B and LLaMA-3-8B)}
\label{sec:exp:commonsense}

\begin{table}[t]
\centering
\small
\setlength{\tabcolsep}{3pt}
\caption{Commonsense reasoning, fine-tuned on Commonsense-170K. }
\label{tab:commonsense_sft}
\begin{tabular}{lrrrrrrrrrr}
\toprule
Method & Trainable \% & BoolQ & PIQA & SIQA & HellaSwag & Wino & ARC-e & ARC-c & OBQA & Avg \\
\midrule
\multicolumn{11}{l}{\emph{\textbf{LLaMA-2-7B}}} \\
Full FT                    & 100.00 & 73.8          & 84.2          & 81.0          & 94.7          & 85.2          & 88.9          & 75.6          & 84.8          & 83.53 \\
LoRA ($r{=}64$)            & 3.33   & 70.8          & 82.8          & 79.4          & 92.9          & 83.4          & 86.3          & 71.6          & 82.8          & 81.25 \\
DoRA ($r{=}64$)            & 3.34   & 71.3          & 83.4          & 80.1          & 92.3          & 84.0          & 86.1          & 71.4          & \textbf{85.8} & 81.80 \\
PiSSA ($r{=}128$)          & 3.33   & 72.5          & 85.3          & 80.8          & 87.2          & 86.1          & 87.1          & 74.3          & 85.6          & 82.36 \\
\method\ (ours)            & 1.53   & \textbf{74.6} & \textbf{86.7} & \textbf{82.5} & \textbf{95.2} & \textbf{86.3} & \textbf{89.0} & \textbf{75.6} & 85.4          & \textbf{84.41} \\
\midrule
\multicolumn{11}{l}{\emph{\textbf{LLaMA-3-8B}}} \\
Full FT                    & 100.00 & 75.4          & 88.0          & 81.8          & 96.5          & 89.3          & 93.1          & 83.0          & 86.0          & 86.64 \\
LoRA ($r{=}64$)            & 1.39   & 71.8          & 85.3          & 80.9          & 93.4          & 84.5          & 90.0          & 77.0          & 84.8          & 83.46 \\
DoRA ($r{=}64$)            & 1.42   & 74.6          & 87.4          & 81.2          & 94.7          & 87.1          & 89.4          & 79.5          & 86.4          & 85.04 \\
PiSSA ($r{=}128$)          & 1.39   & 74.9          & 87.7          & 82.3          & 95.3          & 87.7          & 92.9          & 81.4          & 87.0          & 85.90 \\
MiLoRA ($r{=}64$)          & 1.39   & 73.7          & 87.3          & 80.7          & 94.3          & 87.2          & 91.2          & 78.9          & 86.4          & 84.96 \\
RandLoRA ($r{=}64$)        & 0.33   & 73.1          & 85.5          & 79.4          & 93.1          & 85.9          & 87.9          & 75.9          & 85.4          & 83.28 \\
\method\ (ours)            & 1.46   & \textbf{76.5} & \textbf{90.4} & \textbf{83.0} & \textbf{96.8} & \textbf{89.9} & \textbf{93.7} & \textbf{84.5} & \textbf{89.3} & \textbf{88.01} \\
\bottomrule
\end{tabular}
\vspace{-10pt}
\end{table}

Table~\ref{tab:commonsense_sft} reports per-task accuracy on the eight-task commonsense suite for LLaMA-2-7B and LLaMA-3-8B after training on Commonsense-170K. \method\ achieves $88.01\%$ average accuracy on LLaMA-3-8B, outperforming LoRA by $+4.55$, DoRA by $+2.87$, and Full FT by $+1.37$. Methods that initialize LoRA with principal (PiSSA, $+2.44$ over LoRA) or off-principal (MiLoRA, $+1.50$) singular directions of the pretrained weight improve over LoRA; full-rank-capacity RandLoRA is essentially at LoRA parity ($83.28$ vs.\ $83.46$); while none outperforms Full FT. In contrast, \method\ combines full-rank capacity, preservation of the full-rank pretrained spectral structure, and spectral preconditioning in a unified design, enabling it to surpass Full FT while tuning only $1.46\%$ of parameters.

\subsection{Math Reinforcement Learning}
\label{sec:exp:rl}

\textbf{RL setup.} We run GRPO \cite{shao2024deepseekmath} for $50$ policy-gradient steps with $32$ prompts per step and $8$ rollouts per prompt. The reward is $\{0,1\}$ based on symbolic equivalence of the extracted answer. Rollouts are generated with vLLM \cite{kwon2023vllm}; optimization uses AdamW \cite{loshchilov2019adamw}. AIME-24/25 are reported as avg@8 at temperature $T{=}0.6$; MATH-500 and AMC23 use greedy decoding ($T{=}0$). All runs are on a single NVIDIA H100 (94GB) for training.

\begin{table}[t]
\centering
\small
\setlength{\tabcolsep}{6pt}
\caption{Math RL with GRPO on Qwen3-1.7B and Qwen2.5-7B, $50$ policy-gradient steps.}
\label{tab:rl_paper}
\begin{tabular}{llrrrr}
\toprule
Model & Method & MATH-500 & AMC23 & AIME-24 & AIME-25 \\
\midrule
\multirow{8}{*}{Qwen3-1.7B}
 & Base                                & 52.2          & 30.0          &  5.8          &  7.5          \\
 & Full FT                             & 61.5          & 46.7          & 13.2          & 14.6          \\
 & LoRA \cite{hu2022lora}              & 59.6          & 48.3          &  9.2          & 11.1          \\
 & DoRA \cite{liu2024dora}             & 62.3          & 45.0          & 13.6          & 14.7          \\
 & PiSSA \cite{meng2024pissa}          & 49.1          & 30.8          &  2.2          &  4.7          \\
 & MiLoRA \cite{wang2024milora}        & 58.5          & 41.7          &  7.6          & 11.7          \\
 & RandLoRA \cite{albert2025randlora}  & 62.2          & 55.0          & 13.8          & \textbf{18.3} \\
 & \textbf{\method\ (ours)}            & \textbf{62.5} & \textbf{55.8} & \textbf{13.8} & 13.8          \\
\midrule
\multirow{5}{*}{Qwen2.5-7B}
 & Base                                & 48.4          & 40.0          &  5.0          &  2.9          \\
 & Full FT                             & 58.5          & \textbf{49.2} & 11.0          & 4.9           \\
 & LoRA \cite{hu2022lora}              & 58.5          & \textbf{49.2} & 11.5          & 4.3           \\
 & DoRA \cite{liu2024dora} & 59.3          & 47.5          & \textbf{12.3} & 7.5           \\
 & \textbf{\method\ (ours)}            & \textbf{59.7} & 49.0          & 11.8          & \textbf{7.5}  \\
\bottomrule
\end{tabular}
\vspace{-10pt}
\end{table}

\textbf{Results.} Beyond SFT, \method\ also outperforms other PEFT methods and matches Full FT in reinforcement learning with verifiable rewards (RLVR). Table~\ref{tab:rl_paper} reports results on four held-out mathbenchmarks for Qwen3-1.7B and Qwen2.5-7B, all results are averaged over 3 independent seeds (mean$\pm$std are in Appendix~\ref{sec:appendix:rl_seed_sweep}). On \emph{Qwen3-1.7B}, \method\ outperforms Full FT and other PEFT methods on MATH-500, AMC23, AIME-24. PiSSA and MiLoRA underperform in RLVR due to their low-rank commitment to principal/off-principal spaces, which does not align with RLVR training dynamics \cite{zhu2025rlvrpath, yin2025evaluating}. \method\ preserves full-rank spectrum at a similar parameter budget, thus generalizes better in RLVR. On \emph{Qwen2.5-7B}, \method\ also matches or outperforms Full FT. \method\ also converges faster than Full FT (Figure~\ref{fig:rl_fewer_updates}): compared with the best-tuned Full FT learning rate, \method\ reaches the same MATH eval accuracy with ${\approx}2.5{\times}$ fewer actor updates on Qwen3-1.7B and ${\approx}3.3{\times}$ fewer on Qwen2.5-7B, while Full FT at larger learning rates becomes unstable.

\begin{figure}[t]
  \centering
  \includegraphics[width=\linewidth]{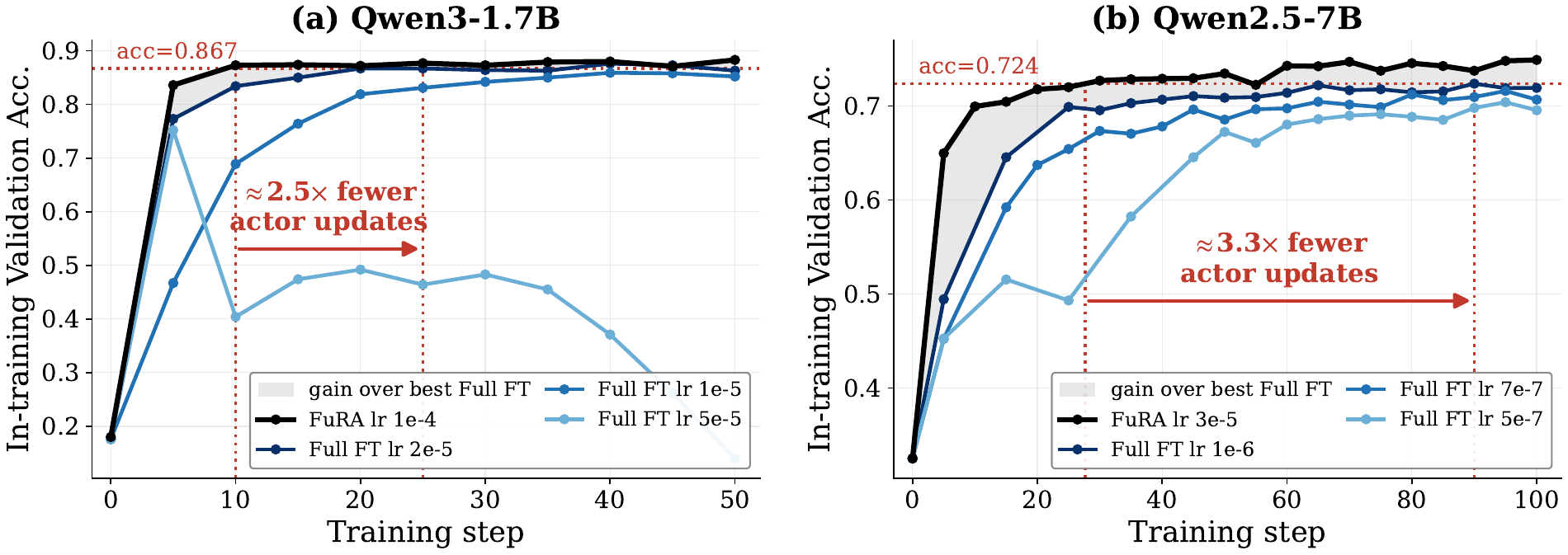}
  \caption{In-training validation accuracy of \method\ and Full FT at several learning rates during GRPO training on \textbf{(a)}~Qwen3-1.7B and \textbf{(b)}~Qwen2.5-7B. \method\ reaches the best Full FT accuracy (dotted horizontal line) with ${\approx}2.5{\times}$ and ${\approx}3.3{\times}$ fewer actor updates, respectively; the shaded area marks the gain over the best Full FT run. The eval protocol of in-training validation differs from that of the final evaluation in Table~\ref{tab:rl_paper}; see Appendix~\ref{sec:appendix:rl_seed_sweep} for details.}
  \label{fig:rl_fewer_updates}
  \vspace{-10pt}
\end{figure}

\subsection{VLM Visual Instruction Tuning (LLaVA-1.5-7B)}
\label{sec:exp:vlm}

We follow the DoRA \cite{liu2024dora} setup to perform visual instruction tuning for LLaVA-1.5-7B \cite{liu2023llava}. Table~\ref{tab:vlm_paper} reports the 7-task average over VQAv2 \cite{goyal2017vqav2}, GQA \cite{hudson2019gqa}, VisWiz \cite{gurari2018vizwiz}, ScienceQA \cite{lu2022scienceqa}, TextVQA \cite{singh2019textvqa}, POPE \cite{li2023pope}, and MMBench \cite{liu2023mmbench}. \method\ exceeds Full FT by $+1.1$ and matches DoRA ($67.6$) with $\mathbf{3.4{\times}}$ fewer trainable parameters ($1.37\%$ vs.\ $4.63\%$). Per-task evaluation results are in Appendix \ref{sec:appendix:vlm}.

\subsection{QFuRA: 4-bit quantized fine-tuning}
\label{sec:exp:qfura}

Following QLoRA \cite{dettmers2023qlora}, we construct QFuRA by quantizing the frozen large core $\mat{L}$ in 4-bit NormalFloat (NF4) while keeping the trainable $\mat{R}$ and $\mat{S}$ in bf16. Tables~\ref{tab:qfura_8b} and~\ref{tab:qfura_70b} report the 8-task Commonsense-170K average on LLaMA-3-8B and GSM8K \cite{cobbe2021gsm8k} accuracy after fine-tuning on MetaMathQA-100K \cite{yu2023metamath} for LLaMA-3-70B. On LLaMA-3-8B, QFuRA reaches $87.30\%$, surpassing QLoRA \cite{dettmers2023qlora} by $+3.41$ and QDoRA \cite{liu2024dora} by $+0.96$ at $1.46\%$ trainable parameters; on LLaMA-3-70B, QFuRA scores $83.78$ on GSM8K, beating QLoRA and QDoRA. The results validate that \method\ is robust to quantizing the frozen large core. Per-task scores are in Appendix~\ref{sec:appendix:qfura}.

\begin{table}[t]
\centering
\small
\begin{minipage}[t]{0.31\linewidth}
\centering
\setlength{\tabcolsep}{4pt}
\caption{Visual instruction tuning of LLaVA-1.5-7B.}
\label{tab:vlm_paper}
\vspace{2pt}
\begin{tabular}{lcc}
\toprule
Method & Params & Avg. \\
\midrule
Full FT                                  & $100\%$           & $66.5$          \\
LoRA \cite{hu2022lora}                  & $4.61\%$          & $66.9$          \\
DoRA \cite{liu2024dora}                 & $4.63\%$          & $67.6$          \\
\textbf{\method\ }    & $1.37\%$ & $\mathbf{67.6}$ \\
\bottomrule
\end{tabular}
\end{minipage}
\hfill
\begin{minipage}[t]{0.31\linewidth}
\centering
\setlength{\tabcolsep}{4pt}
\caption{4-bit quantized fine-tuning on LLaMA-3-8B.}
\label{tab:qfura_8b}
\vspace{2pt}
\begin{tabular}{lcc}
\toprule
Method & Params & Avg. \\
\midrule
QLoRA \cite{dettmers2023qlora}                 & $1.39\%$          & $83.89$          \\
QDoRA \cite{liu2024dora}                & $1.42\%$          & $86.34$          \\
\textbf{QFuRA} & $1.46\%$ & $\mathbf{87.30}$ \\
\bottomrule
\end{tabular}
\end{minipage}
\hfill
\begin{minipage}[t]{0.31\linewidth}
\centering
\setlength{\tabcolsep}{4pt}
\caption{4-bit quantized fine-tuning on LLaMA-3-70B.}
\label{tab:qfura_70b}
\vspace{2pt}
\begin{tabular}{lcc}
\toprule
Method & Params & GSM8K \\
\midrule
QLoRA \cite{dettmers2023qlora}                 & $1.17\%$          & $81.27$          \\
QDoRA \cite{liu2024dora}                 & $1.18\%$          & $81.80$          \\
\textbf{QFuRA} & $1.45\%$ & $\mathbf{83.78}$ \\
\bottomrule
\end{tabular}
\end{minipage}
\vspace{-6pt}
\end{table}

\subsection{Ablations: BlockTT design corners}
\label{sec:exp:ablation}

We ablate \method\ design corners on LLaMA-3-8B Commonsense-170K SFT and Qwen3-1.7B GRPO math RL.  Table~\ref{tab:ablation_peft} compares \method\ Full (all BlockTT factors trainable) and the six \method\ PEFT corners spanning the two design axes of Section~\ref{sec:theory} (Claim~2): block orientation ($m{=}1$ vs.\ $n{=}1$) and placement of $\mat{S}$ (separate trainable, merged into the frozen core, or merged into the trainable core).

Training every factor (\method\ Full) is the best, but the PEFT variant closes nearly the entire gap ($87.91$ vs.\ $88.04$ SFT; matching RL), confirming that the gain comes from full-rank spectral preconditioning rather than additional trainable capacity. Across the six PEFT corners the default preserving output subspace ($m=1$) and keeping singular values as separate parameters ($\fcore{L}\,\tcore{S}\,\tcore{R}$) wins on both SFT and RL. Intuitively, freezing $\mat{L}$ retains the pretrained output feature space while letting the trainable $\mat{R}$ remix based on new input patterns in the fine-tuning data; freezing $\mat{R}$ instead pins the input subspace and could filter out task-specific input features that the fine-tuning data carry. We additionally ablate the BlockTT shape factorization $(n,b)$ of the input dimension and show that balanced factorization achieves better result; results are in Appendix~\ref{sec:appendix:ablations}.

\begin{table}[t]
\centering
\small
\setlength{\tabcolsep}{4pt}
\caption{Ablation study over \method\ design corners. Color code: \fcoretext{blue} = frozen, \tcoretext{red} = trainable.}
\label{tab:ablation_peft}
\begin{tabular}{llcccc}
\toprule
Group & Notation & Train.\ \% & SFT Avg & MATH-500 & AMC23 \\
\midrule
\multirow{2}{*}{Full} & Full FT~~$\tcore{W}$                                  & $100$    & $86.64$          & $63.6$          & $47.5$          \\
& \method\ Full~~$\tcore{L}\,\tcore{S}\,\tcore{R}$      & $100$   & $\mathbf{88.04}$ & \textbf{64.2}      & \textbf{57.5}       \\
\midrule
\multirow{3}{*}{PEFT ($m{=}1$)} & $\fcore{L}\,\tcore{S}\,\tcore{R}$ (default)        & $1.46$ & $\textbf{87.91}$          & $\textbf{63.6}$          & $\textbf{57.5}$          \\
& $\fcore{(LS)}\,\tcore{R}$                          & $1.46$ & $87.12$          & $61.4$          & $55.0$          \\
& $\fcore{L}\,\tcore{(SR)}$                          & $1.46$ & $84.09$          & $62.8$          & $47.5$          \\
\midrule
\multirow{3}{*}{PEFT ($n{=}1$)} & $\tcore{L}\,\tcore{S}\,\fcore{R}$                  & $1.46$ & $86.79$          & $61.2$          & $52.5$          \\
& $\tcore{(LS)}\,\fcore{R}$                          & $1.46$ & $83.30$          & $61.4$          & $42.5$          \\
& $\tcore{L}\,\fcore{(SR)}$                          & $1.46$ & $83.75$          & $62.2$          & $42.5$          \\
\bottomrule
\end{tabular}
\end{table}

\newpage
\section{Conclusion, Limitation, and Broader Impact} \label{sec:conclusion}

We have identified spectral preconditioning as a missing ingredient in current fine-tuning, and have proposed \method, which delivers full-rank update capacity, spectral preconditioning, and LoRA-level parameter efficiency from a single architectural choice. We have demonstrated that in both SFT and RLVR \method\ matches or surpasses Full FT with level step time and GPU memory on par with LoRA. Several directions remain open: a deeper theoretical account of \emph{why} spectral-preconditioned updates improve generalization; initializations from decompositions beyond SVD; custom kernels to further accelerate the block tensor-train contraction; and a stronger Q\method\ that exploits the orthonormal rows of the frozen core $\mat{L}$ for more aggressive quantization. We hope these directions inspire further research on the spectral structure of pretrained models and the design of parameter-efficient adaptation. \method\ could reduce computation and energy cost to access high-quality model adaptation and avoid excessive energy spent on tuning LoRA ranks. It inherits the general societal risks of LLM fine-tuning but introduces no new generative capabilities beyond those of the base checkpoints.

\begin{ack}
This work was supported by an Amazon gift fund.
\end{ack}

\bibliographystyle{abbrv}
\bibliography{neurips_2026}

\newpage
\appendix

\section{\method Algorithm}
\label{sec:algorithm}

\begin{algorithm}[H]
\caption{\method: Full-Rank Adaptation}
\label{alg:fura}
\begin{algorithmic}[1]
\REQUIRE Pretrained weight $\mat{W} \in \mathbb{R}^{d_\text{out} \times d_\text{in}}$, block count $n$, block width $b = d_\text{in}/n$
\STATE Reshape $\mat{W}$ into blocks: $\ten{W} \in \mathbb{R}^{n \times d_\text{out} \times b}$
\FOR{$k = 1, \ldots, n$}
  \STATE $\mat{U}_k, \mat{\Sigma}_k, \mat{V}_k \leftarrow \text{SVD}(\ten{W}_k)$ \hfill $\triangleright$ lossless, $r = b$
\ENDFOR
\STATE $\mat{L}_k \leftarrow \mat{U}_k$, \quad $\mat{S}_k \leftarrow \mathrm{diag}(\mat{\Sigma}_k)$, \quad $\mat{R}_k \leftarrow \mat{V}_k^\top$ \hfill $\triangleright$ assign cores
\STATE Freeze $\mat{L}$; set $\mat{R}$ and $\mat{S}$ as trainable
\WHILE{training}
  \STATE Reshape input: $\ten{X} \in \mathbb{R}^{n \times b}$
  \STATE Forward: $\mat{y} = \sum_{k=1}^{n} \mat{L}_k \, \mathrm{diag}(\mat{S}_k) \, \mat{R}_k \, \ten{X}_k$ \hfill $\triangleright$ two batched GEMMs
  \STATE Update $\mat{R}$ and $\mat{S}$ via optimizer (gradient flows through $\mat{L}$ but $\mat{L}$ is not updated)
\ENDWHILE
\STATE \textbf{Deploy:} merge cores into $\mat{W}' \in \mathbb{R}^{d_\text{out} \times d_\text{in}}$ via Eq.~\eqref{eq:fura} \hfill $\triangleright$ no serving overhead
\end{algorithmic}
\end{algorithm}

\section{\method Complexity analysis}
\label{sec:appendix:complexity}

This section expands the brief complexity discussion in Section~\ref{sec:method} (and the headline numbers in Table~\ref{tab:cost}).

\subsection{System cost measurement protocol}
\label{sec:appendix:syscost}

This section documents the protocol used for the step-time and peak-memory numbers in Table~\ref{tab:cost}.

\paragraph{Setup.} LLaMA-3-8B Commonsense-170K SFT, 300 optimizer steps, single H100 NVL ($95.8$ GB), seed $43$. Per-device batch $8$, gradient accumulation $2$ (effective tokens/step $32{,}768$), sequence length $2048$, bf16 mixed precision, gradient checkpointing on, AdamW with $\eta{=}2{\times}10^{-4}$, linear LR schedule with $3\%$ warmup. Identical for every method; only the adapter recipe varies.

\paragraph{Reported metrics.} \emph{Step (s)} is the median wall-clock per optimizer step over the final $200$ steps (after a $100$-step warmup that excludes adapter-attach and SVD initialization). \emph{Memory (GB)} is \texttt{torch.cuda.max\_memory\_allocated} over the full run, including the warmup window..

\subsection{Results and Analysis}

\method\ matches the per-step time of LoRA/DoRA at comparable budgets, achieves the \emph{lowest} peak GPU memory, and uniquely enables full-rank updates. We detail this below.

\paragraph{Trainable parameters.}
The trainable count is $|\mat{R}| + |\mat{S}| = nrb + nr$. For a square $d\times d$ layer with $n=b=r=\sqrt{d}$, this is $\approx d^{3/2}$, i.e., $\mathcal{O}(1/\sqrt{d})$ of $d^2$. At $d=4096$, $\mat{R}$ has ${\sim}0.26$M parameters, and the overall trainable fraction on LLaMA-3-8B is $1.46\%$, comparable to LoRA/DoRA (${\sim}1.4\%$) and far below LIFT's ${\sim}6.2\%$.

\paragraph{Forward pass.}
For a \method\ layer with $d_\text{in} = n \cdot b$, and rank $r = b$, the cores have shapes $\mat{R} \in \mathbb{R}^{n \times r \times b}$, $\mat{L} \in \mathbb{R}^{n \times d_\text{out} \times r}$, and $\mat{S}=[\mat{S}_1,\dots,\mat{S}_n] \in \mathbb{R}^{n \times r}$ with each $\mat{S}_k \in \mathbb{R}^r$ stored as a vector. The forward pass applies, for input $\mat{x} \in \mathbb{R}^{d_\text{in}}$ reshaped as $\ten{X} \in \mathbb{R}^{n \times b}$,
\begin{align*}
\ten{Z}_k &= \mat{R}_k \ten{X}_k & (k = 1,\dots,n),\quad \ten{Z} \in \mathbb{R}^{n \times r}, \\
\widetilde{\ten{Z}}_k &= \mathrm{diag}(\mat{S}_k) \ten{Z}_k, \\
\mat{y} &= \sum_{k=1}^{n} \mat{L}_k \widetilde{\ten{Z}}_k,\quad \mat{y} \in \mathbb{R}^{d_\text{out}}.
\end{align*}

Table~\ref{tab:cost} reports per-step wall-clock, throughput, and peak GPU memory for $300$-step commonsense SFT on LLaMA-3-8B. \method\ runs at $0.046$ s/step, matching LoRA/DoRA and $4.3\times$ faster than Full FT. Peak memory is $22.7$ GB, lower than LoRA/DoRA and $3.4\times$ below Full FT. Efficiency comes from a sequential design: the $\mat{L}$-stage reduces to a standard dense MVM, while the $\mat{R}$-stage uses fused small MVMs via \texttt{torch.bmm}. With $b=\sqrt{d}=64$, these align well with tensor cores, improving utilization over LoRA's tall-skinny GEMMs.

\paragraph{Memory.}
\method\ achieves lower peak memory than LoRA because its sequential factorization $\mat{L}\,\mat{S}\,\mat{R}\,\mat{x}$ produces a single intermediate activation per layer of size $\mathbb{R}^{n \times r}$, whereas LoRA's parallel residual $\mat{W}\mat{x} + \mat{B}\mat{A}\mat{x}$ stashes \emph{two} additional activations: the down-projection $\mat{A}\mat{x} \in \mathbb{R}^{T \times r}$ and the input $\mat{x} \in \mathbb{R}^{T \times d}$ feeding $\mat{B}\mat{A}\mat{x}$. Compared to sparse fine-tuning method LIFT \cite{liu2025lift}, \method\ uses dramatically less memory because LIFT must additionally keep the full bf16 backbone gradient buffer (${\sim}16$ GB on $8$ B params) plus a per-weight bool mask, whereas \method\ produces its full-rank update entirely through the small factored cores $\mat{R}, \mat{S}$ and never materializes a backbone-sized gradient.


\paragraph{Deployment.}
After training, the three cores $\mat{L},\mat{S},\mat{R}$ merge into a single dense matrix $\mat{W}' \in \mathbb{R}^{d_\text{out} \times d_\text{in}}$ via Eq.~\eqref{eq:fura}, incurring no serving latency relative to the base model.

\paragraph{Initialization cost.}
Block-SVD initialization is a one-time cost: $52.9$ s on LLaMA-3-8B (160 target linear layers), versus $\sim$1 s for LoRA/DoRA. This is $<0.2\%$ of full SFT runtime (10 hours for Commonsense-170K finetuning). The cost can be further reduced by Fast SVD methods like \cite{halko2011finding}.

\section{Hyperparameters}
\label{sec:appendix:exp}

\method\ hyperparameters for Tables~\ref{tab:cost},~\ref{tab:commonsense_sft},~\ref{tab:rl_paper},~\ref{tab:qfura_8b}, and~\ref{tab:qfura_70b}.
Tables~\ref{tab:fura-hparams-commonsense}, \ref{tab:fura-hparams-rl}, \ref{tab:fura-hparams-vlm}, and \ref{tab:fura-hparams-qfura-70b} list the \method\ training configuration used for the four task suites. All \method\ runs use the project default corner: $m=1$ and $\mat{S}$ kept as a separate trainable vector. The LLaMA-3-8B QFuRA row of Table~\ref{tab:qfura_8b} uses the same settings as the LLaMA-3-8B column of Table~\ref{tab:fura-hparams-commonsense}, with the frozen large core $\mat{L}$ NF4-quantized via \texttt{bitsandbytes} and the optimizer swapped to \texttt{PagedAdamW8bit}. The LLaMA-3-70B QFuRA row of Table~\ref{tab:qfura_70b} uses Table~\ref{tab:fura-hparams-qfura-70b}.

\begin{table}[H]
\centering
\small
\setlength{\tabcolsep}{6pt}
\caption{\method\ hyperparameter configuration for the commonsense reasoning tasks (Table~\ref{tab:commonsense_sft}).}
\label{tab:fura-hparams-commonsense}
\begin{tabular}{lcc}
\toprule
Hyperparameters (\method) & LLaMA-2-7B & LLaMA-3-8B \\
\midrule
Dropout                                & $0.0$  & $0.0$  \\
Optimizer                              & AdamW  & AdamW  \\
LR                                     & $3{\times}10^{-4}$ & $2{\times}10^{-4}$ \\
LR Scheduler                           & Linear & Linear \\
Weight decay                           & $0$    & $0$    \\
Batch size (per-device $\times$ grad-accum) & $8\times 2 = 16$ & $8\times 2 = 16$ \\
Warmup ratio                           & $0.03$ & $0.03$ \\
Epochs                                 & $3$    & $3$    \\
Max sequence length                    & $2048$ & $2048$ \\
Mixed precision                        & bf16   & bf16   \\
Gradient checkpointing                 & on     & on     \\
Where                                  & \multicolumn{2}{c}{$Q, K, V, O, \text{Up}, \text{Down}, \text{Gate}$} \\
\bottomrule
\end{tabular}
\end{table}

\begin{table}[H]
\centering
\small
\setlength{\tabcolsep}{6pt}
\caption{\method\ hyperparameter configuration for math GRPO RL (Table~\ref{tab:rl_paper}). Reward $\in\{0,1\}$ via \texttt{math\_utils.is\_equiv}; vLLM rollout at \texttt{gpu\_memory\_utilization}$=0.25$; GRPO clip ratio $0.2$, KL coefficient $0$, no entropy bonus.}
\label{tab:fura-hparams-rl}
\begin{tabular}{lcc}
\toprule
Hyperparameters (\method) & Qwen3-1.7B & Qwen2.5-7B \\
\midrule
Optimizer                              & AdamW & AdamW \\
LR                                     & $1{\times}10^{-4}$ & $1{\times}10^{-4}$ \\
LR Scheduler                           & none (constant) & none (constant) \\
Warmup ratio                           & $0.0$ & $0.0$ \\
Weight decay                           & $0$ & $0$ \\
GRPO steps                             & $50$ & $50$ \\
Prompts per step                       & $32$ & $32$ \\
Group size (rollouts / prompt)         & $8$ & $8$ \\
Epochs per step                        & $1$ & $1$ \\
Max model length (vLLM)                & $2048$ & $2048$ \\
Eval max tokens                        & $2048$ & $2048$ \\
Train rollout temperature              & $1.0$ & $1.0$ \\
Eval temperature                       & \multicolumn{2}{c}{$0.6$ (AIME-24/25), $0.0$ (MATH-500/AMC23)} \\
Seed                                   & $42$ & $42$ \\
Where                                  & \multicolumn{2}{c}{All linear projections (\texttt{trainable\_type=all})} \\
\bottomrule
\end{tabular}
\end{table}

\begin{table}[H]
\centering
\small
\setlength{\tabcolsep}{6pt}
\caption{\method\ hyperparameter configuration for visual instruction tuning of LLaVA-1.5-7B on \texttt{llava\_v1\_5\_mix665k} (5{,}197 steps $=1$ epoch). Same recipe as the LLaVA-1.5 DoRA fine-tune of \cite{liu2024dora} with \method\ substituted for DoRA; best LR selected from $\{2,3,4\}{\times}10^{-4}$.}
\label{tab:fura-hparams-vlm}
\begin{tabular}{lc}
\toprule
Hyperparameters (\method) & LLaVA-1.5-7B \\
\midrule
Dropout                                & $0.05$ \\
Optimizer                              & AdamW \\
LR                                     & $3{\times}10^{-4}$ \\
LR Scheduler                           & Cosine decay \\
Weight decay                           & $0$ \\
Batch size (per-device $\times$ grad-accum) & $4\times 4 = 16$ \\
Warmup ratio                           & $0.03$ \\
Epochs                                 & $1$ \\
Model max length                       & $2048$ \\
Mixed precision                        & bf16 \\
Gradient checkpointing                 & on \\
Where                                  & $Q, K, V, O, \text{Up}, \text{Down}, \text{Gate}$ \\
\bottomrule
\end{tabular}
\end{table}

\begin{table}[H]
\centering
\small
\setlength{\tabcolsep}{6pt}
\caption{Math-10K SFT hyperparameter configuration for the LR / batch-size sweep on LLaMA-3-8B (\S\ref{sec:appendix:math10k}). All three methods share the same training schedule (3 epochs on Math-10K, AdamW, linear LR decay, $0.03$ warmup ratio, max sequence length $2048$, bf16); only LR and effective batch size vary across the sweep. Best LRs found per (method, effective batch size) are reported in \S\ref{sec:appendix:math10k}.}
\label{tab:fura-hparams-math10k}
\begin{tabular}{lccc}
\toprule
Hyperparameters & Full FT & LoRA & \method \\
\midrule
Optimizer                                   & \multicolumn{3}{c}{AdamW} \\
LR Scheduler                                & \multicolumn{3}{c}{Linear} \\
Weight decay                                & \multicolumn{3}{c}{$0$}\\
Warmup ratio                                & \multicolumn{3}{c}{$0.03$}\\
Epochs                                      & \multicolumn{3}{c}{$3$}\\
Max sequence length                         & \multicolumn{3}{c}{$2048$} \\
Mixed precision                             & \multicolumn{3}{c}{bf16} \\
Adapter rank / $\alpha$                     & --- & $r{=}64$, $\alpha{=}128$ & $r{=}\text{full}$, $m{=}1$ \\
Effective batch size (sweep)                & $\{16, 64, 256\}$ & $\{16, 64, 256\}$ & $\{16, 64, 256\}$ \\
LR sweep (bsz $=16$)                        & $\{8, 10, 20, 30\}{\times}10^{-6}$ & $\{3, 6, 10, 20\}{\times}10^{-5}$ & $\{2, 3, 4, 6\}{\times}10^{-4}$ \\
Seed                                        & \multicolumn{3}{c}{$43$} \\
Where                                       & \multicolumn{3}{c}{$Q, K, V, O, \text{Up}, \text{Down}, \text{Gate}$} \\
\bottomrule
\end{tabular}
\end{table}

\begin{table}[H]
\centering
\small
\setlength{\tabcolsep}{6pt}
\caption{QFuRA hyperparameter configuration for math fine-tuning of LLaMA-3-70B on MetaMathQA-100K (Table~\ref{tab:qfura_70b}). Mirrors the QPiSSA recipe ($\eta = 2{\times}10^{-5}$, batch $1{\times}128$, sequence length $512$) with the rank-$r$ SVD adapter replaced by QFuRA's quantized BlockTT factorization. The frozen large core $\mat{L}$ is 4-bit-quantized via \texttt{bitsandbytes} storage layout; the trainable cores $\mat{R}$ and $\mat{S}$ stay in bf16.}
\label{tab:fura-hparams-qfura-70b}
\begin{tabular}{lc}
\toprule
Hyperparameters (QFuRA) & LLaMA-3-70B \\
\midrule
Optimizer                                   & PagedAdamW 8-bit \\
LR                                          & $2{\times}10^{-5}$ \\
LR Scheduler                                & Cosine decay \\
Weight decay                                & $0$ \\
Batch size (per-device $\times$ grad-accum) & $1\times 128 = 128$ \\
Warmup ratio                                & $0.03$ \\
Steps                                       & $100$ \\
Max sequence length                         & $512$ \\
Mixed precision                             & bf16 \\
Trainable param dtype                       & bf16 \\
Seed                                        & $42$ \\
Where                                       & $Q, K, V, O, \text{Up}, \text{Down}, \text{Gate}$ \\
\bottomrule
\end{tabular}
\end{table}

\section{Extended experimental results}
\label{sec:appendix:per_task}

This section collects the per-task scores behind the headline averages reported in the body. Subsection~\ref{sec:appendix:rl_seed_sweep} expands Table~\ref{tab:rl_paper} (Math RL with GRPO on Qwen3-1.7B and Qwen2.5-7B) into per-seed mean$\pm$std across 3 seeds. Subsection~\ref{sec:appendix:qfura} expands Table~\ref{tab:qfura_8b} (QFuRA on LLaMA-3-8B Commonsense-170K) to all eight tasks, Subsection~\ref{sec:appendix:vlm} expands Table~\ref{tab:vlm_paper} (\method\ on LLaVA-1.5-7B visual instruction tuning) to all seven vision-language benchmarks, and Subsection~\ref{sec:appendix:math10k} reports the LR and effective-batch-size sweep for LLaMA-3-8B SFT on Math-10K.

\subsection{LLaMA-3-8B Math-10K SFT: LR and batch-size sweep}
\label{sec:appendix:math10k}

We additionally ran an LR and effective-batch-size sweep on LLaMA-3-8B SFT on Math-10K~\citep{liu2025lift}, evaluated on GSM8K~\citep{cobbe2021gsm8k}. Hyperparameters shared across all runs are in Table~\ref{tab:fura-hparams-math10k}. We sweep the learning rate at the default effective batch size of $16$ for each method: Full FT over $\{8{\times}10^{-6},\, 1{\times}10^{-5},\, 2{\times}10^{-5},\, 3{\times}10^{-5}\}$, LoRA over $\{3{\times}10^{-5},\, 6{\times}10^{-5},\, 1{\times}10^{-4},\, 2{\times}10^{-4}\}$, and \method\ over $\{2,\, 3,\, 4,\, 6\}{\times}10^{-4}$. We then sweep the effective batch size in $\{16, 64, 256\}$, picking the best LR per (method, batch size) by the standard square-root-rule rescaling around the small-batch optimum. Best LRs we found are $\eta=1{\times}10^{-5}$ (Full FT, bsz $16$), $\eta=2{\times}10^{-5}$ (Full FT, bsz $64$ and $256$); $\eta=6{\times}10^{-5}$ (LoRA, bsz $16$), $\eta=2{\times}10^{-4}$ (LoRA, bsz $64$), $\eta=6{\times}10^{-4}$ (LoRA, bsz $256$); $\eta=3{\times}10^{-4}$ (\method, bsz $16$), $\eta=6{\times}10^{-4}$ (\method, bsz $64$), $\eta=8{\times}10^{-4}$ (\method, bsz $256$). Figure~\ref{fig:math10k_sweep} plots the resulting GSM8K accuracy across the LR sweep at bsz $=16$ (left) and across the batch-size sweep with best per-batch LR (right).

\begin{figure}[H]
\centering
\begin{minipage}{0.49\linewidth}
\centering
\includegraphics[width=\linewidth]{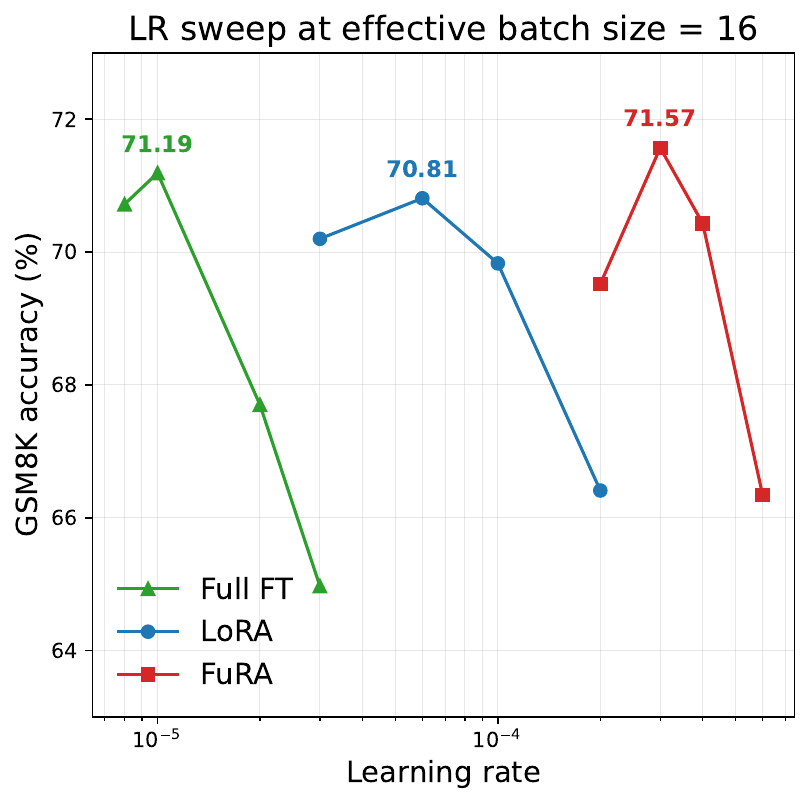}
\end{minipage}\hfill
\begin{minipage}{0.49\linewidth}
\centering
\includegraphics[width=\linewidth]{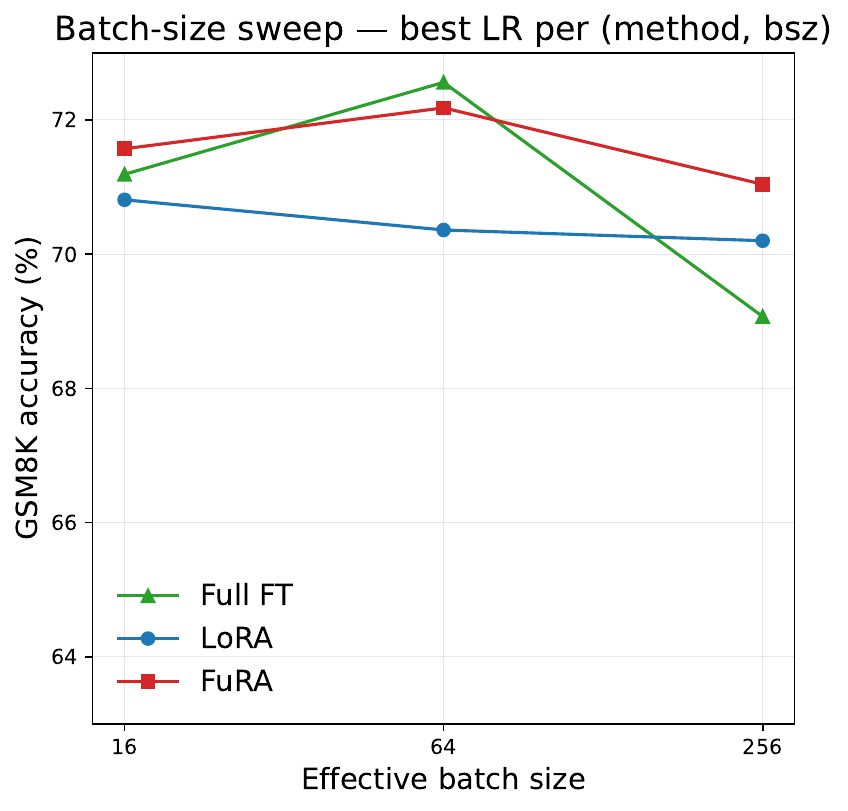}
\end{minipage}
\caption{LLaMA-3-8B Math-10K SFT, GSM8K accuracy. \textbf{Left}: LR sweep at effective batch size $=16$. \textbf{Right}: batch-size sweep. Peak per method is annotated on the left panel.}
\label{fig:math10k_sweep}
\end{figure}

\subsection{LLaMA-2-7B and LLaMA-3-8B Commonsense-170K Results}

\paragraph{Source of the baseline numbers.}
For Tables~\ref{tab:commonsense_sft}, the \method\ result is the averaged result across 3 random seeds, with specified numbers per seed in Table \ref{tab:fura-seed-sweep}.  \method\ result has low variance across different random seeds.      
The LLaMA-2-7B rows and the LLaMA-3-8B Full FT / LoRA / DoRA rows, together with the LIFT comparison numbers are reproduced from the hyperparameter setting in LIFT paper~\citep{liu2025lift}, which itself uses the LLM-Adapters~\citep{hu2023llmadapters} recipe of $r{=}64$ adapters 
on five linear modules \texttt{q\_proj k\_proj v\_proj up\_proj down\_proj}, 
best learning rate for each method, and others aligned with Tables~\ref{tab:fura-hparams-commonsense}. Tuning all 7 modules downgrades the LoRA performance \cite{hu2023llmadapters}. We additionally ran the \emph{RandLoRA, MiLoRA, and PiSSA} baselines on LLaMA-3-8B in-house under the same recipe and 8-task evaluation harness, sweeping the learning rate per method. RandLoRA and MiLoRA peak at $\eta = 1{\times}10^{-4}$; PiSSA peaks at $\eta = 2{\times}10^{-5}$. 

\begin{table}[H]                                                                                                                                                                                                                 
  \centering                                                                                                                                                                                                                       
  \caption{FuRA (PEFT) seed sweep on LLaMA-3-8B fine-tuned on Commonsense-170K (3 epochs).                    }                                                                                                                                                                                       
  \label{tab:fura-seed-sweep}                                                                                                                                                                                                      
  \small                                                                                                                                                                                                                           
  \setlength{\tabcolsep}{4pt}                                                                                                                                                                                                      
  \begin{tabular}{lcccccccc|c}
  \toprule                                                                                                                                                                                                                         
  Seed & BoolQ & PIQA & SIQA & HellaSwag & Wino & ARC-e & ARC-c & OBQA & Avg \\
  \midrule                                                                                                                                                                                                                         
  42       & 76.30 & 90.60 & 83.00 & 96.80 & 90.10 & 93.80 & 85.40 & 89.00 & 88.13 \\
  43       & 76.60 & 89.90 & 83.20 & 96.80 & 89.70 & 93.60 & 84.10 & 89.40 & 87.91 \\                                                                                                                                              
  44       & 76.50 & 90.80 & 82.80 & 96.90 & 89.80 & 93.60 & 84.00 & 89.60 & 88.00 \\                                                                                                                                              
  \midrule                                                                                                                                                                                                                         
  Mean     & 76.47 & 90.43 & 83.00 & 96.83 & 89.87 & 93.67 & 84.50 & 89.33 & 88.01 \\                                                                                                                                              
  Std      &  0.12 &  0.39 &  0.16 &  0.05 &  0.17 &  0.09 &  0.64 &  0.25 &  0.09 \\                                                                                                                                              
  \bottomrule                                                                                                                                                                                                                      
  \end{tabular}                                                                                                                                                                                                                    
  \end{table} 

\subsection{Math RL: per-seed mean$\pm$std across 3 seeds}
\label{sec:appendix:rl_seed_sweep}

This subsection expands Table~\ref{tab:rl_paper} into per-seed mean$\pm$std across three independent seeds (42, 43, 44) for each (model, method) cell. Sample std is taken across $n{=}3$ seeds, so the SEM $\approx 0.6\times$ std. AIME-24/25 are avg@8 at $T{=}0.6$; MATH-500 and AMC23 are greedy@1 at $T{=}0$. The Paper Table (Table~\ref{tab:rl_paper}) reports the per-seed mean from the same set of runs.

\begin{table}[H]
\centering
\small
\setlength{\tabcolsep}{4pt}
\caption{Math RL with GRPO: per-seed mean$\pm$std across 3 seeds (42, 43, 44). The Paper Table (Table~\ref{tab:rl_paper}) reports the per-seed mean of the same runs.}
\label{tab:rl_seed_sweep}
\begin{tabular}{llrrrr}
\toprule
Model & Method & MATH-500 & AMC23 & AIME-24 & AIME-25 \\
\midrule
\multirow{7}{*}{Qwen3-1.7B}
 & Full FT  & $61.53\pm2.10$ & $46.67\pm3.82$  & $13.21\pm0.66$ & $14.58\pm1.10$ \\
 & LoRA     & $59.60\pm0.92$ & $48.33\pm7.64$  & $ 9.15\pm2.90$ & $11.10\pm0.63$ \\
 & DoRA     & $62.33\pm2.58$ & $45.00\pm10.90$ & $13.60\pm4.42$ & $14.73\pm0.24$ \\
 & PiSSA    & $49.07\pm5.56$ & $30.83\pm10.10$ & $ 2.22\pm1.27$ & $ 4.71\pm4.34$ \\
 & MiLoRA   & $58.53\pm0.50$ & $41.67\pm3.82$  & $ 7.63\pm0.86$ & $11.67\pm1.81$ \\
 & RandLoRA & $62.23\pm0.58$ & $55.00\pm2.50$  & $13.83\pm1.03$ & $18.30\pm1.46$ \\
 & \method\ & $62.47\pm1.33$ & $55.83\pm1.44$  & $13.75\pm1.10$ & $13.75\pm1.50$ \\
\midrule
\multirow{4}{*}{Qwen2.5-7B}
 & Full FT             & $58.47\pm2.53$ & $49.17\pm5.77$  & $10.97\pm2.69$ & $4.87\pm1.96$ \\
 & LoRA                & $58.47\pm0.23$ & $49.17\pm10.10$ & $11.52\pm1.22$ & $4.30\pm1.47$ \\
 & DoRA    & $59.30\pm0.42$ & $47.50\pm3.54$  & $12.29\pm2.67$ & $7.49\pm1.16$ \\
 & \method\            & $59.73\pm0.99$ & $49.00\pm2.50$  & $11.81\pm1.20$ & $7.50\pm2.54$ \\
\bottomrule
\end{tabular}
\end{table}

Figure~\ref{fig:rl_fewer_updates} tracks in-training validation accuracy, which uses a different protocol from the final evaluation in Table~\ref{tab:rl_paper}. For Qwen3-1.7B, the GRPO loop evaluates every $5$ steps with greedy decoding ($T{=}0$, at most $1024$ new tokens) on $1{,}000$ held-out MATH problems (all Algebra, disjoint from the RL training prompts), and scores a response by matching its last \texttt{\textbackslash boxed\{\}} answer to the reference. For Qwen2.5-7B, we report MATH-500 accuracy averaged over $8$ sampled responses per problem (mean@8). Table~\ref{tab:rl_paper} instead evaluates the final checkpoint on the full MATH-500 with greedy decoding \@ 1, up to $2048$ new tokens.

\subsection{QFuRA: per-task scores on Commonsense-170K (LLaMA-3-8B)}
\label{sec:appendix:qfura}

\begin{table}[H]
\centering
\small
\setlength{\tabcolsep}{2pt}
\caption{4-bit QFuRA vs.\ QLoRA, QDoRA, and bf16 full fine-tuning on LLaMA-3-8B Commonsense-170K. \textbf{Bold} = column-best across the 4-bit quantized methods.}
\label{tab:qfura_commonsense}
\begin{tabular}{lrrrrrrrrrr}
\toprule
Method & Trainable \% & BoolQ & PIQA & SIQA & HellaSwag & Wino & ARC-e & ARC-c & OBQA & Avg \\
\midrule
Full FT (bf16)                    & $100.00$        & $75.4$          & $88.0$          & $81.8$          & $96.5$          & $89.3$          & $93.1$          & $83.0$          & $86.0$          & $86.64$ \\
\midrule
4-bit QLoRA ($r{=}64$)             & $2.05$          & $72.7$          & $86.6$          & $80.8$          & $93.7$          & $85.4$          & $89.9$          & $76.2$          & $85.8$          & $83.89$ \\
4-bit QDoRA ($r{=}64$)             & $2.06$          & $\mathbf{75.5}$ & $88.0$          & $81.1$          & $95.9$          & $88.4$          & $91.8$          & $82.0$          & $88.0$          & $86.34$ \\
\textbf{4-bit QFuRA (ours)}        & $\mathbf{1.46}$ & $73.0$          & $\mathbf{89.9}$ & $\mathbf{82.7}$ & $\mathbf{96.6}$ & $\mathbf{89.1}$ & $\mathbf{93.1}$ & $\mathbf{83.4}$ & $\mathbf{90.6}$ & $\mathbf{87.30}$ \\
\bottomrule
\end{tabular}
\end{table}

Table \ref{tab:qfura_8b} (full results Table \ref{tab:qfura_commonsense}) are all our own runs as specified in Tables~\ref{tab:fura-hparams-commonsense}.
QFuRA wins seven of eight per-task columns among the 4-bit methods (all but BoolQ, where QDoRA leads by $2.5$pp), and the gain over QDoRA is concentrated on ARC-c ($+1.4$), ARC-e ($+1.3$), OBQA ($+2.6$), and PIQA ($+1.9$). QFuRA's training and evaluation hyperparameters are documented in Table~\ref{tab:fura-hparams-commonsense} (LLaMA-3-8B column) with the only delta being NF4 quantization of $\mat{L}$ via \texttt{bitsandbytes} and the \texttt{PagedAdamW8bit} optimizer.

\subsection{VLM visual instruction tuning: per-task scores (LLaVA-1.5-7B)}
\label{sec:appendix:vlm}

We follow the exact LLaVA-1.5 \cite{liu2023llava} + DoRA \cite{liu2024dora} fine-tuning recipe: 1 epoch on \texttt{llava\_v1\_5\_mix665k} ($5{,}197$ optimizer steps), per-device batch $4$ with gradient accumulation $4$ (effective batch $16$), sequence length $2048$, bf16, gradient checkpointing on, AdamW, cosine LR schedule with $3\%$ warmup, and \texttt{mm\_projector\_lr}$=2{\times}10^{-5}$. \method\ replaces DoRA on all seven linear projections of the language tower (Q, K, V, O, Gate, Up, Down) under the project default corner (\texttt{output\_one\_block} / small core trainable / \texttt{s\_merged\_to=keep\_trainable} / full rank), with the CLIP-L/336 vision encoder and the MLP projector kept in their LLaVA-1.5 recipe. We swept the \method\ LR over $\{2,3,4\}{\times}10^{-4}$; LR$=3{\times}10^{-4}$ wins the 7-task average and is the row reported below.

\textbf{Evaluation protocol.} We follow the DoRA paper's Table 12 conventions: VQAv2 = test-dev2015 Overall (server eval); GQA = testdev\_balanced; VisWiz = test Overall (server eval); SQA = full Acc on all $21$K questions; VQA$^{\text{T}}$ = TextVQA val Acc; POPE = mean F1 across $\{$random, popular, adversarial$\}{\times}100$; MMBench = dev split, CircularEval. Full FT, LoRA, and DoRA columns are quoted from \cite{liu2024dora} Table 12.

\begin{table}[H]
\centering
\small
\setlength{\tabcolsep}{4pt}
\caption{Per-task LLaVA-1.5-7B visual instruction tuning results on the 7 vision-language benchmarks of \cite{liu2024dora}, Table 12. Full FT / LoRA / DoRA quoted from the DoRA paper; \method\ is our run at LR$=3{\times}10^{-4}$ (best of $\{2,3,4\}{\times}10^{-4}$). \textbf{Bold} = column-best.}
\label{tab:vlm_full}
\begin{tabular}{lccccccccc}
\toprule
Method & \# Params (\%) & VQAv2 & GQA & VisWiz & SQA & VQA$^{\text{T}}$ & POPE & MMBench & Avg. \\
\midrule
Full FT                  & $100$           & $78.5$ & $61.9$ & $50.0$ & $66.8$ & $58.2$ & $85.9$ & $64.3$ & $66.5$ \\
LoRA \cite{hu2022lora}   & $4.61$          & $\mathbf{79.1}$ & $62.9$ & $47.8$ & $68.4$ & $58.2$ & $86.4$ & $\mathbf{66.1}$ & $66.9$ \\
DoRA \cite{liu2024dora}  & $4.63$          & $78.6$ & $\mathbf{62.9}$ & $52.2$ & $\mathbf{69.9}$ & $57.0$ & $\mathbf{87.2}$ & $\mathbf{66.1}$ & $\mathbf{67.6}$ \\
\textbf{\method\ (ours)} & $\mathbf{1.37}$ & $78.7$ & $62.7$ & $\mathbf{54.4}$ & $67.3$ & $\mathbf{58.1}$ & $86.6$ & $64.5$ & $\mathbf{67.6}$ \\
\bottomrule
\end{tabular}
\end{table}

\textbf{Takeaways.} \method\ matches DoRA's 7-task average while training $\mathbf{3.4{\times}}$ fewer parameters ($1.37\%$ vs.\ $4.63\%$ of $7.06$B model parameters). \method\ wins VisWiz ($+2.2$) and TextVQA ($+1.1$), ties POPE within $0.6$, and trails on the reasoning-heavy SQA ($-2.6$) and MMBench ($-1.6$) splits. The wall-clock cost is also lower: at the same effective batch and step count, \method\ trains in $21$h$11$m on a single H100 vs.\ DoRA's $35$h$46$m ($\mathbf{41\%}$ faster).

\textbf{Param count derivation ($1.37\%$).} For each linear with \texttt{in\_features} $=n\cdot b$, the trainable count under the default \method\ corner is $\texttt{btt\_r}+\texttt{btt\_s} = n b^2 + n b = \texttt{in\_features}\cdot(b{+}1)$, while the dense $\texttt{btt\_l}\in\mathbb{R}^{\texttt{in\_features}\times\texttt{out\_features}}$ core stays frozen. With $32$ LLaMA-2-7B blocks $\times\,7$ projections per block, $\texttt{in\_features}\in\{4096, 11008\}$ factor as $(64,64)$ and $(86,128)$ respectively, giving $32\cdot(6\cdot 4096\cdot 65 + 11008\cdot 129)\approx 96.6$M trainable parameters out of LLaVA-1.5-7B's $7.06$B total ($\approx 95.4$M from $\texttt{btt\_r}$ plus $1.1$M from $\texttt{btt\_s}$, with embeddings, lm\_head, norms, the CLIP vision tower, and the MLP projector all kept frozen): $96.6/7060=1.37\%$.

\subsection{Ablation details}
\label{sec:appendix:ablations}

The body of Section~\ref{sec:exp:ablation} reports only the SFT 8-task average and the MATH-500/AMC23 RL columns. Table \ref{tab:ablation_peft_app_rl} below give the corresponding per-task SFT scores and Table \ref{tab:ablation_peft_app_rl} adds the AIME-24/AIME-25 RL columns for both ablation slices. Color code matches the body: \fcoretext{blue} = frozen, \tcoretext{red+underbar} = trainable.

\begin{table}[H]
\centering
\small
\setlength{\tabcolsep}{3pt}
\caption{\method\ PEFT design corners, full SFT per-task results (LLaMA-3-8B Commonsense-170K).}
\label{tab:ablation_peft_app_sft}
\begin{tabular}{llcccccccccc}
\toprule                                                                                                                       Group & Notation & Train.\ \% & BoolQ & PIQA & SIQA & HSwag & Wino & ARCe & ARCc & OBQA & Avg \\
  \midrule                                                                                              
  \multirow{2}{*}{Full} & Full FT~~$\tcore{W}$                              & 100 & 75.40 & 88.00          & 81.80 & 96.50          & 89.30 & 93.10          & 83.00          & 86.00 & 86.64          \\
  & \method\ $\tcore{L}\,\tcore{S}\,\tcore{R}$                         &  100 & 76.10 & \textbf{90.80} & 82.90 & \textbf{96.80} & 89.30 & \textbf{94.40} & \textbf{85.20} & 88.80 & \textbf{88.04} \\    
\midrule
\multirow{3}{*}{PEFT ($m{=}1$)} & $\fcore{L}\,\tcore{S}\,\tcore{R}$ (default)        & $1.46$ & $\mathbf{76.6}$ & $89.9$         & $\mathbf{83.2}$ & $\mathbf{96.8}$ & $\mathbf{89.7}$ & $\mathbf{93.6}$ & $\mathbf{84.1}$ & $\mathbf{89.4}$ & $\mathbf{87.91}$ \\
& $\fcore{(LS)}\,\tcore{R}$                          & $1.46$ & $74.0$         & $\mathbf{90.8}$ & $81.9$         & $95.9$         & $89.4$         & $93.4$         & $84.0$         & $87.6$         & $87.12$         \\
& $\fcore{L}\,\tcore{(SR)}$                          & $1.46$ & $74.8$         & $76.6$         & $82.2$         & $95.4$         & $86.4$         & $90.9$         & $81.6$         & $84.8$         & $84.09$         \\
\midrule
\multirow{3}{*}{PEFT ($n{=}1$)} & $\tcore{L}\,\tcore{S}\,\fcore{R}$                  & $1.46$ & $75.0$         & $89.5$         & $81.9$         & $96.2$         & $88.1$         & $93.2$         & $83.0$         & $87.4$         & $86.79$         \\
& $\tcore{(LS)}\,\fcore{R}$                          & $1.46$ & $67.0$         & $85.6$         & $80.6$         & $93.5$         & $86.3$         & $90.2$         & $78.4$         & $84.8$         & $83.30$         \\
& $\tcore{L}\,\fcore{(SR)}$                          & $1.46$ & $66.3$         & $85.6$         & $79.5$         & $94.7$         & $86.1$         & $90.5$         & $80.5$         & $86.8$         & $83.75$         \\
\bottomrule
\end{tabular}
\end{table}

\begin{table}[H]
\centering
\small
\setlength{\tabcolsep}{4pt}
\caption{\method\ PEFT design corners, full RL benchmarks (Qwen3-1.7B GRPO math).}
\label{tab:ablation_peft_app_rl}
\begin{tabular}{llccccc}
\toprule                                                  
  Group & Notation & Train.\ \% & MATH-500 & AMC23 & AIME-24 & AIME-25 \\
  \midrule                                                                          
  \multirow{2}{*}{Full} & Full FT~~$\tcore{W}$                              & 100 & \textbf{63.6} & 47.5          & 13.8          & 15.4          \\
  & \method\ $\tcore{L}\,\tcore{S}\,\tcore{R}$                         &  100 & 62.4          & \textbf{50.0} & \textbf{17.1} & \textbf{18.3} \\
\midrule
\multirow{3}{*}{PEFT ($m{=}1$)} & $\fcore{L}\,\tcore{S}\,\tcore{R}$ (default)        & $1.46$ & $\mathbf{63.6}$ & $\mathbf{57.5}$ & $\mathbf{15.0}$ & $\mathbf{17.5}$ \\
& $\fcore{(LS)}\,\tcore{R}$                          & $1.46$ & $61.4$         & $55.0$         & $10.0$         & $12.9$         \\
& $\fcore{L}\,\tcore{(SR)}$                          & $1.46$ & $62.8$         & $47.5$         & $12.5$         & $16.7$         \\
\midrule
\multirow{3}{*}{PEFT ($n{=}1$)} & $\tcore{L}\,\tcore{S}\,\fcore{R}$                  & $1.46$ & $61.2$         & $52.5$         & $9.6$          & $13.8$         \\
& $\tcore{(LS)}\,\fcore{R}$                          & $1.46$ & $61.4$         & $42.5$         & $12.9$         & $13.8$         \\
& $\tcore{L}\,\fcore{(SR)}$                          & $1.46$ & $62.2$         & $42.5$         & $9.6$          & $8.3$          \\
\bottomrule
\end{tabular}
\end{table}

\subsubsection*{Shape factorization ablation}
\label{sec:appendix:ablation_shape}

Holding the s-placement recipe fixed at the headline default $\fcore{L}\,\tcore{S}\,\tcore{R}$, this sub-ablation varies how the input dimension $d_{\text{in}}$ is split into block sizes $(n, b)$ with $n\cdot b = d_{\text{in}}$. We evaluate three regimes on LLaMA-3-8B Commonsense-170K (1 epoch, $\eta=2{\times}10^{-4}$, seed 43) covering the full balanced-to-extreme axis.

\begin{table}[H]
\centering
\small
\setlength{\tabcolsep}{3pt}
\caption{Shape factorization ablation: how the input dimension $d_{\text{in}}=n\cdot b$ is split into BlockTT block sizes. LLaMA-3-8B, Commonsense-170K, 1 epoch, $\eta=2{\times}10^{-4}$, seed 43, default s-placement $\fcore{L}\,\tcore{S}\,\tcore{R}$. \textbf{Bold} = column-best.}
\label{tab:ablation_shape}
\begin{tabular}{lcccccccccc}
\toprule
Shape variant & Train.\ \% & BoolQ & PIQA & SIQA & HSwag & Wino & ARCe & ARCc & OBQA & Avg \\
\midrule
Default ($b{\approx}\sqrt{d_{\text{in}}}$) & $1.46$ & $\mathbf{75.70}$ & $\mathbf{90.80}$ & $\mathbf{83.80}$ & $\mathbf{96.80}$ & $\mathbf{88.70}$ & $\mathbf{93.70}$ & $\mathbf{84.60}$ & $\mathbf{90.40}$ & $\mathbf{88.06}$ \\
Unbalanced ($b{=}8$)                                    & $0.03$ & $68.00$          & $86.20$          & $74.10$          & $91.20$          & $75.70$          & $90.20$          & $77.10$          & $77.40$          & $79.99$          \\
Extreme ($b{=}1$, $n{=}d_{\text{in}}$)                                     & $0.14$ & $35.80$          & $47.70$          & $32.90$          & $24.50$          & $47.40$          & $25.30$          & $24.20$          & $27.60$          & $33.18$          \\
\bottomrule
\end{tabular}
\end{table}

Concretely on LLaMA-3-8B: Default uses $(n{=}32, b{=}128)$ for $Q/K/V/O$ (head-aligned, $n{=}\text{num\_heads}$, $b{=}\text{head\_dim}$), $(n{=}64, b{=}64)$ for Gate/Up, and $(n{=}112, b{=}128)$ for Down (closest factor pair). Unbalanced uses $(n{=}512, b{=}8)$ for $Q/K/V/O$ and Gate/Up, and $(n{=}1792, b{=}8)$ for Down. Extreme uses $b{=}1, n{=}d_{\text{in}}$ on every module.

A more balanced shape factorization is preferred: pushing $b$ toward $1$ (large $n$) hurts model quality: Avg drops from $88.06$ (balanced) to $79.99$ ($b{=}8$) to $33.18$ ($b{=}1$, near chance level on most tasks). Under $\fcore{L}\,\tcore{S}\,\tcore{R}$ the trainable count is dominated by $|\mat{R}|=r\cdot d_{\text{in}}$ with effective rank $r=\min(a,b)$; shrinking $b$ shrinks both the achievable rank and the small-core size, so very small $b$ is a double penalty on representational capacity \emph{and} trainable budget. The Extreme row ends up with more trainable parameters than the Unbalanced row (\,$0.14\%$ vs.\ $0.03\%$\,) because the separately-held $\tcore{S}$ vector dominates the parameter count once $r{=}1$, yet model quality is worse.

\section{Extended proof}
\subsection{Properties of the energy ratio \eqref{eq:col_energy_ratio}}
In the following lemma, we summarize the properties of the energy ratio \eqref{eq:col_energy_ratio} to support our observations on the spectral analysis for the Full FT.
\begin{lemma}[Properties of the Energy Ratio]\label{lem:properties_energy_ratio}
Let $\mat{U} \in \mathbb{R}^{d_{\mathrm{out}} \times k}$ have orthonormal columns. Then:
\begin{enumerate}
    \item[\textbf{(i)}] \textbf{(Boundedness and geometric interpretation).} For any $\mat{M} \in \mathbb{R}^{d_{\mathrm{out}} \times d_{\mathrm{in}}}$ with $\|\mat{M}\|_F > 0$,
    \[
        0 \;\leq\; \rho(\mat{M};\, \mat{U}) \;\leq\; 1\,.
    \]
    Moreover:
    \begin{itemize}
        \item $\rho(\mat{M}; \mat{U}) = 1$ if and only if $\mathrm{col}(\mat{M}) \subseteq \mathrm{col}(\mat{U})$, i.e., every column of $\mat{M}$ lies in the column space of $\mat{U}$.
        \item $\rho(\mat{M}; \mat{U}) = 0$ if and only if $\mathrm{col}(\mat{M}) \perp \mathrm{col}(\mat{U})$, i.e., every column of $\mat{M}$ is orthogonal to the column space of $\mat{U}$.
        \item $\rho(\mat{M}; \mat{U}) \in (0, 1)$ if and only if $\mat{M}$ spreads its energy between $\mathrm{col}(\mat{U})$ and its orthogonal complement.
    \end{itemize}

    \item[\textbf{(ii)}] \textbf{(Gaussian baseline).} If $\mat{G} \in \mathbb{R}^{d_{\mathrm{out}} \times d_{\mathrm{in}}}$ has i.i.d.\ entries with zero mean and variance $\sigma^2 > 0$, then
    \begin{equation}
        \mathbb{E}\bigl[\rho(\mat{G};\, \mat{U})\bigr] \;=\; \frac{k}{d_{\mathrm{out}}}\,.
    \end{equation}
\end{enumerate}
\end{lemma}

\begin{proof}
\textbf{Part (i).}

Write $\mat{M} = [\mat{m}_1, \ldots, \mat{m}_{d_{\mathrm{in}}}]$ column-wise. Since $\mat{U}$ has orthonormal columns, $\mat{P} = \mat{U}\mat{U}^\top$ is the orthogonal projector onto $\mathrm{col}(\mat{U})$. Note that $\|\mat{U}^\top \mat{m}_j\|^2 = \|\mat{P}\mat{m}_j\|^2$ for each column. We decompose each column as
\[
    \mat{m}_j = \mat{P}\mat{m}_j + (\mat{I} - \mat{P})\mat{m}_j\,,
\]
where the two components are orthogonal. By the Pythagorean theorem:
\[
    \|\mat{m}_j\|^2 = \|\mat{P}\mat{m}_j\|^2 + \|(\mat{I}-\mat{P})\mat{m}_j\|^2\,.
\]
Summing over columns:
\[
    \|\mat{M}\|_F^2 = \|\mat{U}^\top \mat{M}\|_F^2 + \|(\mat{I}-\mat{P})\mat{M}\|_F^2\,.
\]
Since both terms are non-negative and their sum equals $\|\mat{M}\|_F^2 > 0$, we have $0 \leq \|\mat{U}^\top \mat{M}\|_F^2 \leq \|\mat{M}\|_F^2$, giving $\rho(\mat{M}; \mat{U}) \in [0, 1]$.

\emph{Case $\rho = 1$:} $\rho(\mat{M}; \mat{U}) = 1$ iff $\|(\mat{I}-\mat{P})\mat{M}\|_F^2 = 0$ iff $(\mat{I}-\mat{P})\mat{m}_j = \mat{0}$ for all $j$ iff $\mat{m}_j \in \mathrm{col}(\mat{U})$ for all $j$ iff $\mathrm{col}(\mat{M}) \subseteq \mathrm{col}(\mat{U})$.

\emph{Case $\rho = 0$:} $\rho(\mat{M}; \mat{U}) = 0$ iff $\|\mat{U}^\top \mat{M}\|_F^2 = 0$ iff $\mat{U}^\top \mat{m}_j = \mat{0}$ for all $j$ iff $\mat{m}_j \perp \mathrm{col}(\mat{U})$ for all $j$ iff $\mathrm{col}(\mat{M}) \perp \mathrm{col}(\mat{U})$.

\emph{Case $\rho \in (0,1)$:} This is the complement of the above two cases: $\mat{M}$ has nonzero projection onto both $\mathrm{col}(\mat{U})$ and $\mathrm{col}(\mat{U})^\perp$, meaning its energy is distributed between the two subspaces.

\bigskip
\noindent\textbf{Part (ii).}

Let $\mat{G}$ have i.i.d.\ entries with mean zero and variance $\sigma^2$. Denote the columns of $\mat{U}$ by $\mat{u}_1, \ldots, \mat{u}_k$.

\emph{Numerator:}
\begin{align}
    \mathbb{E}\bigl[\|\mat{U}^\top \mat{G}\|_F^2\bigr]
    &= \sum_{j=1}^{d_{\mathrm{in}}} \sum_{i=1}^{k} \mathbb{E}\bigl[(\mat{u}_i^\top \mat{g}_j)^2\bigr]\,.
\end{align}
Each $\mat{u}_i^\top \mat{g}_j = \sum_{\ell=1}^{d_{\mathrm{out}}} (\mat{u}_i)_\ell\, G_{\ell j}$ is a linear combination of independent zero-mean entries, so
\[
    \mathbb{E}[(\mat{u}_i^\top \mat{g}_j)^2] = \mathrm{Var}(\mat{u}_i^\top \mat{g}_j) = \sigma^2 \sum_{\ell=1}^{d_{\mathrm{out}}} (\mat{u}_i)_\ell^2 = \sigma^2 \|\mat{u}_i\|^2 = \sigma^2\,.
\]
Therefore $\mathbb{E}[\|\mat{U}^\top \mat{G}\|_F^2] = k \cdot d_{\mathrm{in}} \cdot \sigma^2$.

\emph{Denominator:}
\[
    \mathbb{E}\bigl[\|\mat{G}\|_F^2\bigr] = d_{\mathrm{out}} \cdot d_{\mathrm{in}} \cdot \sigma^2\,.
\]

\emph{Computing $\mathbb{E}[\rho(\mat{G}; \mat{U})]$:}
Let $\mat{P} = \mat{U}\mat{U}^\top$ as in Part~(i). By the Pythagorean decomposition established there:
\[
    \|\mat{G}\|_F^2 = \|\mat{U}^\top \mat{G}\|_F^2 + \|(\mat{I} - \mat{P})\mat{G}\|_F^2\,.
\]
Therefore
\[
    \rho(\mat{G}; \mat{U}) = \frac{\|\mat{U}^\top \mat{G}\|_F^2}{\|\mat{U}^\top \mat{G}\|_F^2 + \|(\mat{I} - \mat{P})\mat{G}\|_F^2}\,.
\]
Extend $\mat{U}$ to a full orthonormal basis $\bar{\mat{U}} = [\mat{U} \mid \mat{U}_\perp] \in \mathbb{R}^{d_{\mathrm{out}} \times d_{\mathrm{out}}}$, where $\mat{U}_\perp \in \mathbb{R}^{d_{\mathrm{out}} \times (d_{\mathrm{out}} - k)}$. Define $\mat{H} = \bar{\mat{U}}^\top \mat{G}$. Since $\bar{\mat{U}}$ is orthogonal and $\mat{G}$ has i.i.d.\ $\mathcal{N}(0, \sigma^2)$ entries, $\mat{H}$ also has i.i.d.\ $\mathcal{N}(0, \sigma^2)$ entries (orthogonal transformations preserve the i.i.d.\ Gaussian distribution). The first $k$ rows of $\mat{H}$ equal $\mat{U}^\top \mat{G}$ and the remaining $d_{\mathrm{out}} - k$ rows equal $\mat{U}_\perp^\top \mat{G} = (\mat{I} - \mat{P})\mat{G}$ expressed in the $\mat{U}_\perp$ basis. Thus:
\[
    \|\mat{U}^\top \mat{G}\|_F^2 = \sum_{i=1}^{k} \sum_{j=1}^{d_{\mathrm{in}}} H_{ij}^2\,, \qquad
    \|\mat{G}\|_F^2 = \sum_{i=1}^{d_{\mathrm{out}}} \sum_{j=1}^{d_{\mathrm{in}}} H_{ij}^2\,.
\]
Let $X_\ell = H_\ell^2 / \sigma^2$ for each entry $\ell$. All $X_\ell$ are i.i.d.\ $\chi^2(1)$. Then:
\[
    \rho(\mat{G}; \mat{U}) = \frac{\sum_{\ell=1}^{k \cdot d_{\mathrm{in}}} X_\ell}{\sum_{\ell=1}^{d_{\mathrm{out}} \cdot d_{\mathrm{in}}} X_\ell}\,.
\]
The numerator is a partial sum of the denominator, and all terms are i.i.d. By symmetry, each term contributes equally in expectation to the total sum. Therefore:
\[
    \mathbb{E}\!\left[\frac{\sum_{\ell=1}^{k \cdot d_{\mathrm{in}}} X_\ell}{\sum_{\ell=1}^{d_{\mathrm{out}} \cdot d_{\mathrm{in}}} X_\ell}\right] = \frac{k \cdot d_{\mathrm{in}}}{d_{\mathrm{out}} \cdot d_{\mathrm{in}}} = \frac{k}{d_{\mathrm{out}}}\,,
\]
where the first equality follows from the exchangeability of the $\{X_\ell\}$: for i.i.d.\ positive random variables $X_1, \ldots, X_N$, the expectation $\mathbb{E}[X_1 / (X_1 + \cdots + X_N)] = 1/N$ by symmetry, so $\mathbb{E}[(\sum_{\ell=1}^n X_\ell) / (\sum_{\ell=1}^N X_\ell)] = n/N$.
\end{proof}
\subsection{Claim 1: Full-rank Update Capacity}\label{sec:appendix:proofs}

We follow \cite{zhang2024spectral} to define the rank capacity of an weight update $\Delta W=f_{\theta}(W)$ as
\begin{equation}
    \mathcal{R}(f_{\theta}; W) := \max_{\theta} \,\mathrm{rank}\!\left(f_{\theta}(W)\right) \;-\; \min_{\theta} \,\mathrm{rank}\!\left(f_{\theta}(W)\right)
\end{equation}

which describes the range of matrix ranks the weight update can achieve.

\begin{lemma}[Rank capacity of slice-wise spectral-core adaptation]
\label{lem:rank-formal}
Let $\mat{W} \in \mathbb{R}^{m\times n}$ be a full-rank pretrained weight matrix with $\rank(\mat{W})=\min(m,n)$.

\textbf{(Column-sliced form).} Suppose $n=n_0 n_1$, and write
\[
\mat{W} = [\mat{W}_1,\dots,\mat{W}_{n_1}], \quad \mat{W}_i \in \mathbb{R}^{m\times n_0}.
\]
For each slice, consider a full-rank factorization
\[
\mat{W}_i = \mat{P}_i \mat{Q}_i,
\quad
\mat{P}_i \in \mathbb{R}^{m\times r_c},\;
\mat{Q}_i \in \mathbb{R}^{r_c\times n_0},
\quad
r_c = \min(m,n_0).
\]
Fix $\mat{P}_i$ and tune $\{\mat{Q}_i\}$. Define
\[
f_{\theta}(\mat{W})
=
[\mat{P}_1 \widetilde{\mat{Q}}_1,\dots,\mat{P}_{n_1} \widetilde{\mat{Q}}_{n_1}].
\]
Then
\[
\mathcal{R}(f_\theta;\mat{W})
=
\rank([\mat{P}_1,\dots,\mat{P}_{n_1}]).
\]
In particular, if $m \le n$, then $\mathcal{R}(f_\theta;\mat{W}) = m$.

\textbf{(Row-sliced form).} Suppose $m=m_0 m_1$, and write
\[
\mat{W} =
\begin{bmatrix}
\mat{W}_1 \\ \vdots \\ \mat{W}_{m_1}
\end{bmatrix},
\quad
\mat{W}_j \in \mathbb{R}^{m_0\times n}.
\]
For each slice, consider a full-rank factorization
\[
\mat{W}_j = \mat{U}_j \mat{V}_j,
\quad
\mat{U}_j \in \mathbb{R}^{m_0\times r_r},\;
\mat{V}_j \in \mathbb{R}^{r_r\times n},
\quad
r_r = \min(m_0,n).
\]
Fix $\mat{V}_j$ and tune $\{\mat{U}_j\}$. Define
\[
g_{\phi}(\mat{W})
=
\begin{bmatrix}
\widetilde{\mat{U}}_1 \mat{V}_1 \\
\vdots \\
\widetilde{\mat{U}}_{m_1} \mat{V}_{m_1}
\end{bmatrix}.
\]
Then
\[
\mathcal{R}(g_\phi;\mat{W})
=
\rank\!\left(
\begin{bmatrix}
\mat{V}_1^\top & \cdots & \mat{V}_{m_1}^\top
\end{bmatrix}
\right).
\]
In particular, if $n \le m$, then $\mathcal{R}(g_\phi;\mat{W}) = n$.
\end{lemma}

\begin{proof}
We first prove the column-sliced case. By construction,
\[
f_\theta(\mat{W})
=
[\mat{P}_1 \widetilde{\mat{Q}}_1,\dots,\mat{P}_{n_1} \widetilde{\mat{Q}}_{n_1}].
\]
For each $i$, every column of $\mat{P}_i \widetilde{\mat{Q}}_i$ lies in $\mathrm{col}(\mat{P}_i)$. Hence
\[
\mathrm{col}(f_\theta(\mat{W}))
\subseteq
\mathrm{col}([\mat{P}_1,\dots,\mat{P}_{n_1}]),
\]
which implies
\[
\rank(f_\theta(\mat{W}))
\le
\rank([\mat{P}_1,\dots,\mat{P}_{n_1}]).
\]
Thus
\[
\max_\theta \rank(f_\theta(\mat{W}))
\le
\rank([\mat{P}_1,\dots,\mat{P}_{n_1}]).
\]

On the other hand, since each $\widetilde{\mat{Q}}_i \in \mathbb{R}^{r_c\times n_0}$ is unconstrained, we can choose $\widetilde{\mat{Q}}_i$ so that the columns of $\mat{P}_i \widetilde{\mat{Q}}_i$ span any subspace contained in $\mathrm{col}(\mat{P}_i)$. Aggregating across all slices, we can realize any column set in
\[
\mathrm{col}([\mat{P}_1,\dots,\mat{P}_{n_1}]).
\]
Therefore,
\[
\max_\theta \rank(f_\theta(\mat{W}))
=
\rank([\mat{P}_1,\dots,\mat{P}_{n_1}]).
\]

Moreover, by setting $\widetilde{\mat{Q}}_i = \mat{0}$ for all $i$, we obtain $f_\theta(\mat{W}) = \mat{0}$, hence
\[
\min_\theta \rank(f_\theta(\mat{W})) = 0.
\]
Thus,
\[
\mathcal{R}(f_\theta;\mat{W})
=
\rank([\mat{P}_1,\dots,\mat{P}_{n_1}]).
\]

Now assume $m \le n$ and $\rank(\mat{W}) = m$. Since
\[
\mat{W} = [\mat{P}_1 \mat{Q}_1,\dots,\mat{P}_{n_1} \mat{Q}_{n_1}],
\]
all columns of $\mat{W}$ lie in $\mathrm{col}([\mat{P}_1,\dots,\mat{P}_{n_1}])$, so
\[
m = \rank(\mat{W})
\le
\rank([\mat{P}_1,\dots,\mat{P}_{n_1}]).
\]
Since the latter is at most $m$, we conclude
\[
\rank([\mat{P}_1,\dots,\mat{P}_{n_1}]) = m,
\]
and hence $\mathcal{R}(f_\theta;\mat{W}) = m$.

The row-sliced case follows symmetrically. For
\[
g_\phi(\mat{W})
=
\begin{bmatrix}
\widetilde{\mat{U}}_1 \mat{V}_1 \\
\vdots \\
\widetilde{\mat{U}}_{m_1} \mat{V}_{m_1}
\end{bmatrix},
\]
each row lies in $\mathrm{row}(\mat{V}_j)$. Hence
\[
\mathrm{row}(g_\phi(\mat{W}))
\subseteq
\operatorname{span}\{\mathrm{row}(\mat{V}_j)\}_{j=1}^{m_1},
\]
which implies
\[
\max_\phi \rank(g_\phi(\mat{W}))
=
\rank\!\left(
\begin{bmatrix}
\mat{V}_1^\top & \cdots & \mat{V}_{m_1}^\top
\end{bmatrix}
\right).
\]
Again the minimum rank is $0$. If $n \le m$ and $\rank(\mat{W})=n$, then the row space must have dimension $n$, yielding
\[
\mathcal{R}(g_\phi;\mat{W}) = n.
\]
\end{proof}

\paragraph{Constraints of block-wise fixed-subspace adaptation.}
Although the proposed slice-wise parameterization attains maximal rank capacity, it imposes strong structural constraints on the adapted weight. We characterize these constraints below.

\begin{proposition}[Block-wise fixed-subspace constraint]
\label{prop:block-constraint-col}
Consider the column-sliced parameterization
\[
f_\theta(\mat{W}) = [\mat{P}_1 \widetilde{\mat{Q}}_1,\dots,\mat{P}_{n_1} \widetilde{\mat{Q}}_{n_1}],
\]
with fixed $\{\mat{P}_i\}$. Then for any $\theta$:
\[
\mathrm{col}\!\left(f_\theta(\mat{W})_{(:,\mathcal{I}_i)}\right)
\subseteq
\mathrm{col}(\mat{P}_i),
\]
where $\mathcal{I}_i$ indexes columns belonging to slice $i$. Equivalently,
\[
f_\theta(\mat{W}) \in \mathcal{S}_{\text{col}}
:=
\left\{
[\mat{X}_1,\dots,\mat{X}_{n_1}]
\;\middle|\;
\mathrm{col}(\mat{X}_i)\subseteq \mathrm{col}(\mat{P}_i)
\right\}.
\]
\end{proposition}

\begin{proposition}[Row-space constraint (dual form)]
\label{prop:block-constraint-row}
For the row-sliced parameterization
\[
g_\phi(\mat{W})
=
\begin{bmatrix}
\widetilde{\mat{U}}_1 \mat{V}_1 \\
\vdots \\
\widetilde{\mat{U}}_{m_1} \mat{V}_{m_1}
\end{bmatrix},
\]
we have
\[
\mathrm{row}\!\left(g_\phi(\mat{W})_{(\mathcal{J}_j,:)}\right)
\subseteq
\mathrm{row}(\mat{V}_j),
\]
and hence
\[
g_\phi(\mat{W}) \in \mathcal{S}_{\text{row}}
:=
\left\{
\begin{bmatrix}
\mat{Y}_1\\ \vdots \\ \mat{Y}_{m_1}
\end{bmatrix}
\;\middle|\;
\mathrm{row}(\mat{Y}_j)\subseteq \mathrm{row}(\mat{V}_j)
\right\}.
\]
\end{proposition}

\paragraph{Implications.}
The above propositions show that adaptation is restricted to a \emph{product of subspaces} across slices. This leads to several nontrivial consequences:

\textbf{(1) No cross-block subspace transfer.} Columns in slice $i$ cannot utilize directions in $\mathrm{col}(\mat{P}_j)$ for $j\neq i$. Hence global redistribution of information across slices is prohibited.

\textbf{(2) Restricted perturbation geometry.} Let $\Delta \mat{W} = f_\theta(\mat{W}) - \mat{W}$. Then
\[
\Delta \mat{W}
=
[\mat{P}_1 \Delta \mat{Q}_1,\dots,\mat{P}_{n_1} \Delta \mat{Q}_{n_1}],
\]
so each block update lies in a fixed linear subspace. In particular,
\[
\Delta \mat{W} \in \mathcal{T}
:=
\left\{
[\mat{P}_1 \mat{Z}_1,\dots,\mat{P}_{n_1} \mat{Z}_{n_1}]
\right\},
\]
which is a strict subset of $\mathbb{R}^{m\times n}$ unless all $\mathrm{col}(\mat{P}_i)$ span $\mathbb{R}^m$ individually.

\textbf{(3) Coupled global rank vs local expressivity.} Although the global rank can reach $\rank([\mat{P}_1,\dots,\mat{P}_{n_1}])$, each slice individually satisfies
\[
\rank(\mat{P}_i \widetilde{\mat{Q}}_i) \le r_c.
\]
Hence expressivity is locally bottlenecked even when global rank is maximal.

\textbf{(4) Subspace preservation bias.} The adaptation preserves the pretrained column subspaces $\{\mathrm{col}(\mat{P}_i)\}$. Therefore, learning is biased toward \emph{reweighting and recombining existing features} rather than creating new directions.

\paragraph{Comparison to LoRA and spectral adapters.}
Unlike LoRA, which introduces a global low-rank update $\Delta \mat{W} = \mat{A}\mat{B}^\top$ spanning all columns, the above parameterization enforces block-wise independence of subspaces. Compared to spectral adapters that directly modify singular directions, this approach fixes the singular subspaces and only adapts coefficients within them.

\paragraph{When is this favorable?}
This constraint is beneficial when pretrained subspaces are already well-aligned with downstream tasks, as it:
(i) preserves stable directions,
(ii) reduces destructive interference,
and (iii) improves conditioning by restricting updates to structured subspaces.

\subsection{Preconditioner derivations}
\label{sec:appendix:precond}

Let $\mat{W}_k = \mat{U}_k \mat{\Sigma}_k \mat{V}_k^\top$ be the SVD of one block. Let $\mathcal{L}$ be the loss, $\mat{y} = \mat{W}\mat{x}$, $\mat{g} = \partial \mathcal{L}/\partial \mat{y}$. Standard identities give the unconstrained per-block gradient $\partial \mathcal{L}/\partial \mat{W}_k = \mat{g}\mat{x}_k^\top$. We specialize to the four \method\ variants.

\paragraph{Corner 1: output-side training, $\mat{S}$ merged to frozen $\mat{R}$.} The parameterization is $\mat{W}_k = \mat{L}_k \mat{R}_k$ with $\mat{L}_k$ trainable and $\mat{R}_k = \mathrm{diag}(\mat{S}_k)\,\mat{V}_k^\top$ frozen. The gradient on $\mat{L}_k$ is
\[ \partial \mathcal{L}/\partial \mat{L}_k \;=\; (\mat{g}\mat{x}_k^\top) \mat{R}_k^\top \;=\; \mat{g}\mat{x}_k^\top\,\mat{V}_k\,\mathrm{diag}(\mat{S}_k). \]
A GD step $\mat{L}_k \gets \mat{L}_k - \eta\, \mat{g}\mat{x}_k^\top\,\mat{V}_k\,\mathrm{diag}(\mat{S}_k)$ yields
\[ \Delta \mat{W}_k \;=\; \Delta \mat{L}_k \cdot \mat{R}_k \;=\; -\eta\, \mat{g}\mat{x}_k^\top\,\mat{V}_k\,\mathrm{diag}(\mat{S}_k)^2\,\mat{V}_k^\top, \]
the ``$\mathrm{diag}(\mat{S})^2$-weighted'' row of Table~\ref{tab:preconditioners}: $\mathrm{diag}(\mat{S}_k)$ enters once via $\mat{R}_k^\top$ in the gradient and once via $\mat{R}_k$ in the reconstruction.

\paragraph{Corner 2: output-side training, $\mat{S}$ merged to trainable $\mat{L}$.} The parameterization is $\mat{W}_k = \mat{L}_k\mat{R}_k$ with $\mat{L}_k = \mat{U}_k\,\mathrm{diag}(\mat{S}_k)$ trainable and $\mat{R}_k = \mat{V}_k^\top$ frozen. The gradient on $\mat{L}_k$ is
\[ \partial \mathcal{L}/\partial \mat{L}_k \;=\; (\mat{g}\mat{x}_k^\top)\mat{R}_k^\top \;=\; \mat{g}\mat{x}_k^\top\,\mat{V}_k. \]
A GD step yields $\Delta\mat{L}_k = -\eta\,\mat{g}\mat{x}_k^\top\,\mat{V}_k$, hence
\[ \Delta \mat{W}_k \;=\; \Delta\mat{L}_k\cdot\mat{R}_k \;=\; -\eta\, \mat{g}\mat{x}_k^\top\,\mat{V}_k\mat{V}_k^\top, \]
the ``fair'' projector $\mat{V}_k\mat{V}_k^\top$ row of Table~\ref{tab:preconditioners}: when $\mat{S}$ is absorbed into the trainable side, its scaling is freely re-tuned by the optimizer and drops out of the effective preconditioner, leaving an unweighted projection onto $\row(\mat{V}_k^\top)$.

\paragraph{Corner 3 (default): input-side training, $\mat{S}$ kept separate trainable.} Parameterize $\mat{W}_k = \mat{L}_k\,\mathrm{diag}(\mat{S}_k)\,\mat{R}_k$ with $\mat{L}_k = \mat{U}_k$ frozen and $\mat{R}_k = \mat{V}_k^\top$, $\mat{S}_k$ trainable. The gradient on $\mat{R}_k$ is $\partial\mathcal{L}/\partial\mat{R}_k = \mathrm{diag}(\mat{S}_k)\,\mat{U}_k^\top\,\mat{g}\mat{x}_k^\top$; the gradient on $\mat{S}_k$ is the diagonal of $\mat{U}_k^\top\,\mat{g}\mat{x}_k^\top\,\mat{V}_k$. A GD step on $\mat{R}_k$ gives $\Delta\mat{R}_k = -\eta\,\mathrm{diag}(\mat{S}_k)\,\mat{U}_k^\top\,\mat{g}\mat{x}_k^\top$, so the effective weight update from $\mat{R}$ is
\[ \Delta \mat{W}_k\big|_{\mat{R}} \;=\; \mat{U}_k\,\mathrm{diag}(\mat{S}_k)\,\Delta\mat{R}_k \;=\; -\eta\, \mat{U}_k\,\mathrm{diag}(\mat{S}_k)^2\,\mat{U}_k^\top\,\mat{g}\mat{x}_k^\top, \]
where $\mathrm{diag}(\mat{S}_k)$ enters twice: once from $\mat{L}_k\,\mathrm{diag}(\mat{S}_k)$ multiplying $\Delta\mat{R}_k$, and once inside $\Delta\mat{R}_k$ itself. Combined with the $\mat{S}$ update:
\[ \Delta \mat{W}_k \;=\; -\eta\, \mat{U}_k\,\mathrm{diag}(\mat{S}_k)^2\,\mat{U}_k^\top\,\mat{g}\mat{x}_k^\top \;+\; \mat{U}_k\,\mathrm{diag}(\Delta \mat{S}_k)\,\mat{V}_k^\top, \]
the default preconditioner of Equation~\eqref{eq:default-precond}. The first term combines column-space projection ($\mat{U}_k\mat{U}_k^\top$, non-trivial since $\mat{U}_k$ is rectangular) with singular-value preconditioning ($\mathrm{diag}(\mat{S}_k)^2$ scaling).

\paragraph{Corner 4: input-side training, $\mat{S}$ merged to frozen $\mat{L}$.} The parameterization is $\mat{W}_k = \mat{L}_k\mat{R}_k$ with $\mat{L}_k = \mat{U}_k\,\mathrm{diag}(\mat{S}_k)$ frozen and $\mat{R}_k = \mat{V}_k^\top$ trainable. The gradient on $\mat{R}_k$ is
\[ \partial \mathcal{L}/\partial \mat{R}_k \;=\; \mat{L}_k^\top(\mat{g}\mat{x}_k^\top) \;=\; \mathrm{diag}(\mat{S}_k)\,\mat{U}_k^\top\,\mat{g}\mat{x}_k^\top. \]
A GD step yields $\Delta\mat{R}_k = -\eta\,\mathrm{diag}(\mat{S}_k)\,\mat{U}_k^\top\,\mat{g}\mat{x}_k^\top$, hence
\[ \Delta \mat{W}_k \;=\; \mat{L}_k\cdot\Delta\mat{R}_k \;=\; -\eta\, \mat{U}_k\,\mathrm{diag}(\mat{S}_k)^2\,\mat{U}_k^\top\,\mat{g}\mat{x}_k^\top, \]
the principal-biased $\mat{U}_k\,\mathrm{diag}(\mat{S}_k)^2\,\mat{U}_k^\top$ row of Table~\ref{tab:preconditioners} (PiSSA-like): $\mathrm{diag}(\mat{S}_k)$ enters once via $\mat{L}_k^\top$ in the gradient and once via $\mat{L}_k$ in the reconstruction, mirroring Corner 1 with the input/output sides swapped. The remaining row of Table~\ref{tab:preconditioners} (input-side training with $\mat{S}$ merged into the trainable $\mat{R}$) follows by the same calculation as Corner 2 with the input/output sides swapped, yielding the unweighted projector $\mat{U}_k\mat{U}_k^\top$.

\section{Full per-layer sweeps for the motivating figures}
\label{sec:appendix:motivation_sweeps}

This section provides the full per-layer / per-module versions of the panels shown for a single representative weight matrix in Figures~\ref{fig:motivation_svdft} and~\ref{fig:motivation_fura} of the body. The body uses layer~15 \texttt{q\_proj} as a representative slice; the figures here verify that the trends hold across all transformer layers and across the matched set of linear modules.

\begin{figure}[H]
  \centering
  \includegraphics[width=\linewidth]{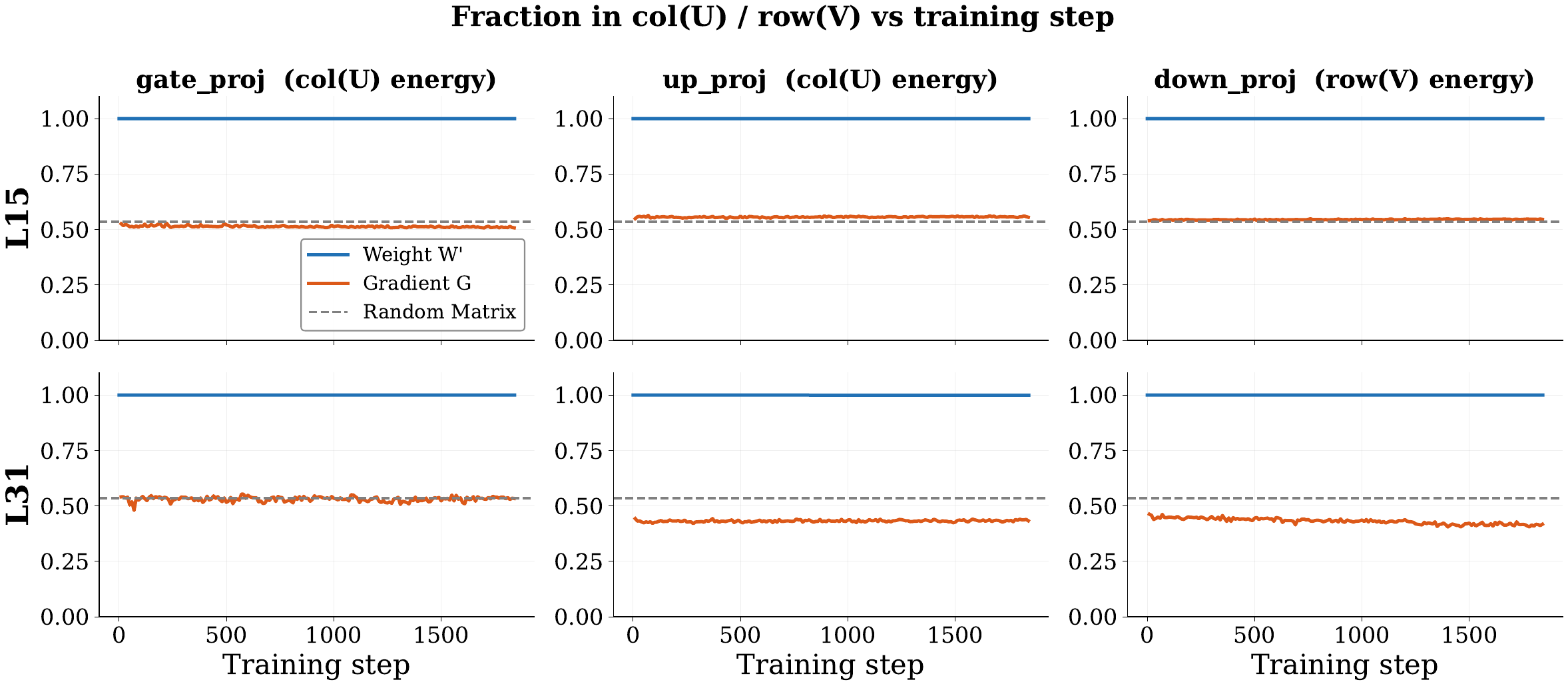}
  \caption{Full per-layer / per-module sweep for Figure~\ref{fig:motivation_svdft}\textbf{(a)}: energy-ratio $\rho(\mat{G}_{\mat{W}};\mat{U})$ and $\rho(\mat{W}';\mat{U})$ tracked through Full FT training on LLaMA-3-8B / Math-10K. The body figure shows layer~15 \texttt{q\_proj}; the same pattern (gradient near the random baseline, weight near $1.0$) holds across all layers and modules.}
  \label{fig:motivation_svdft_a_sweep}
\end{figure}

\begin{figure}[H]
  \centering
  \includegraphics[width=\linewidth]{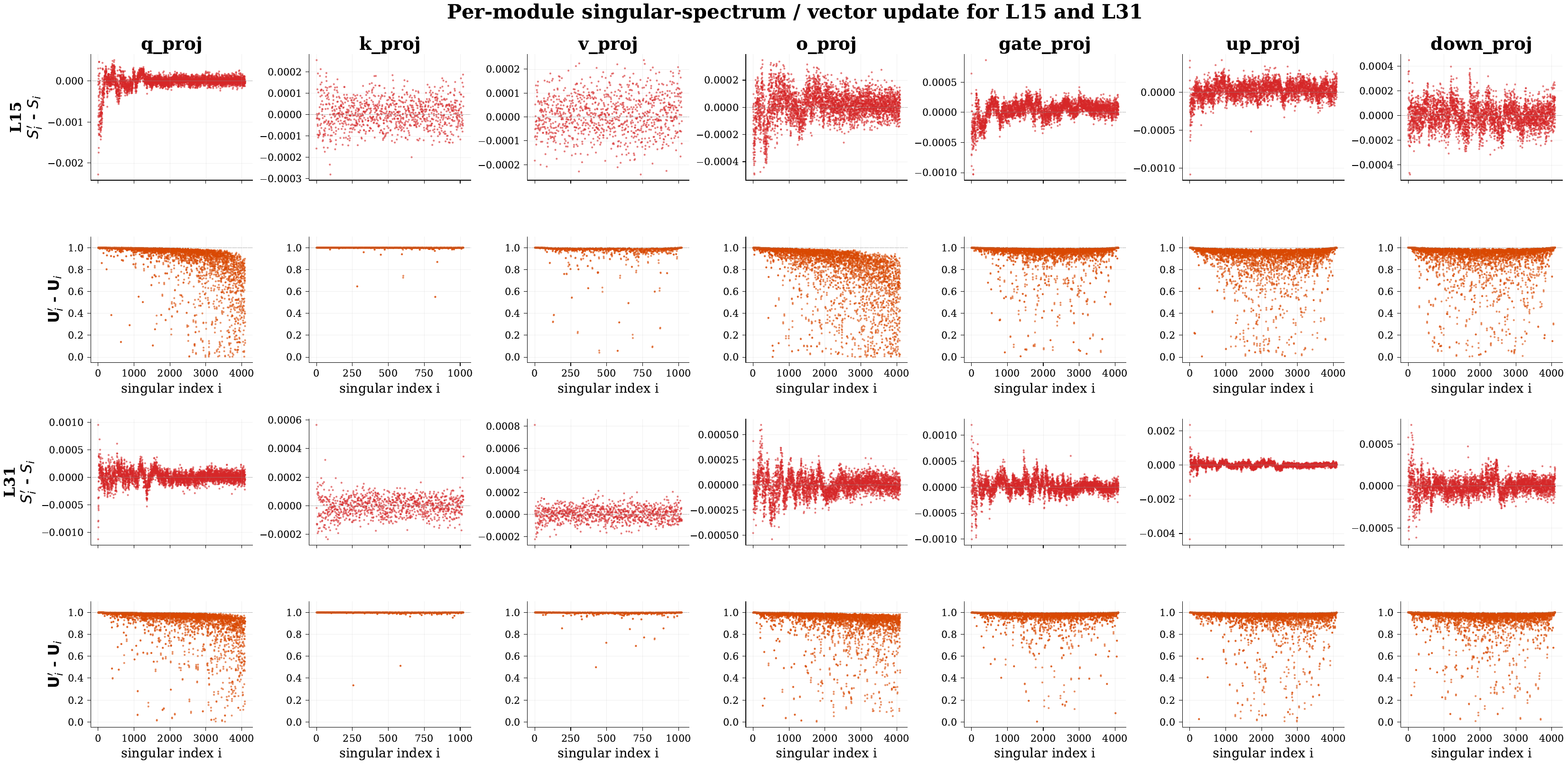}
  \caption{Full per-layer / per-module sweep for Figure~\ref{fig:motivation_svdft}\textbf{(b)}: comparison of singular values and singular vectors of $\mat{W}$ versus $\mat{W}'$ after Full FT on LLaMA-3-8B / Math-10K. The body figure shows layer~15 \texttt{q\_proj}; across all layers and modules only a few singular values change significantly while most remain close to their pretrained values.}
  \label{fig:motivation_svdft_b_sweep}
\end{figure}

\begin{figure}[H]
  \centering
  \includegraphics[width=\linewidth]{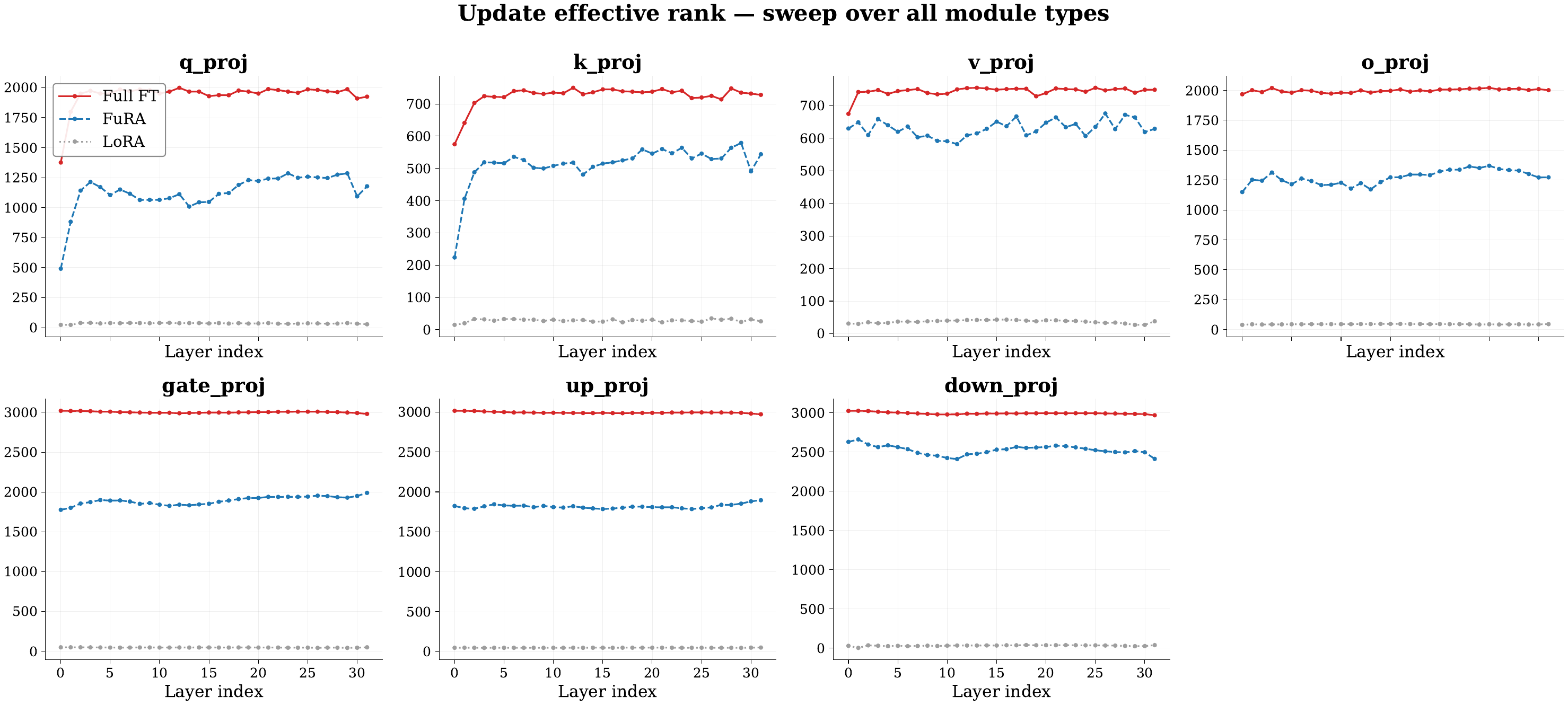}
  \caption{Full per-layer / per-module sweep for Figure~\ref{fig:motivation_fura}\textbf{(a)}: effective rank of $\Delta\mat{W}$ on the Qwen3-1.7B GRPO math-RL task. The body figure shows a representative module; here we report all modules in every transformer layer, confirming that \method's effective rank closely tracks Full FT layer-by-layer while remaining well above LoRA.}
  \label{fig:motivation_fura_a_sweep}
\end{figure}

\begin{figure}[H]
  \centering
  \includegraphics[width=\linewidth]{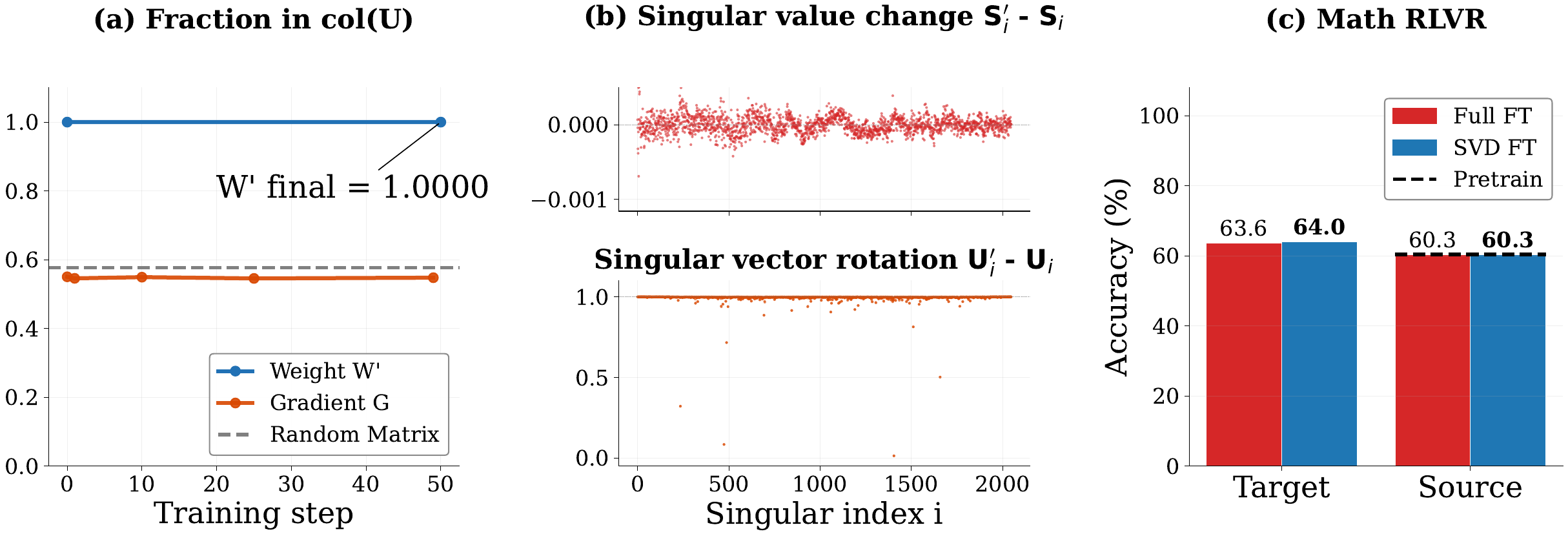}
  \caption{RLVR counterpart of Figure~\ref{fig:motivation_svdft} on Qwen3-1.7B GRPO math RL (layer~14 \texttt{up\_proj}). \textbf{(a)}~The gradient $\mat{G}_{\mat{W}}$ stays near the random-matrix energy ratio, while the fine-tuned weight $\mat{W}'$ remains fully inside $\col(\mat{U})$. \textbf{(b)}~Singular values change only slightly and most singular vectors stay aligned with their pretrained counterparts after $50$ GRPO steps. \textbf{(c)}~SVD FT matches or exceeds Full FT on the target domain (MATH-500) and on the source domain (HellaSwag), and both stay at the pretrained source-domain level.}
  \label{fig:motivation_svdft_rl}
\end{figure}

\newpage
\input{checklist.tex}

\end{document}

%% file: checklist.tex
\section*{NeurIPS Paper Checklist}

\begin{enumerate}

\item {\bf Claims}
    \item[] Question: Do the main claims made in the abstract and introduction accurately reflect the paper's contributions and scope?
    \item[] Answer: \answerYes{}
    \item[] Justification: Empirical claims in the abstract/introduction (LLaMA-3-8B Commonsense $88.01$, $+1.37$ over Full FT; matching/surpassing Full FT on Qwen3-1.7B and Qwen2.5-7B math RLVR; $+1.1$ over Full FT on LLaVA-1.5-7B; QFuRA over QLoRA) are supported by Tables~\ref{tab:commonsense_sft}, \ref{tab:rl_paper}, \ref{tab:vlm_paper}, \ref{tab:qfura_8b}, \ref{tab:qfura_70b} in Section~\ref{sec:exp}; theoretical claims are supported in Section~\ref{sec:theory}.

\item {\bf Limitations}
    \item[] Question: Does the paper discuss the limitations of the work performed by the authors?
    \item[] Answer: \answerYes{}
    \item[] Justification: Section~\ref{sec:conclusion} discusses open directions and limitations: the lack of a deeper theoretical account of \emph{why} spectral preconditioning improves generalization, restriction to SVD-based initialization, lack of custom kernels for the block tensor-train contraction, and the unexplored quantization headroom of QFuRA. The remark following Proposition~\ref{prop:fullrank} also flags the cross-block locality of the default $m{=}1$ orientation as the primary expressivity cost.

\item {\bf Theory assumptions and proofs}
    \item[] Question: For each theoretical result, does the paper provide the full set of assumptions and a complete (and correct) proof?
    \item[] Answer: \answerYes{}
    \item[] Justification: Proposition~\ref{prop:fullrank} and the preconditioner identity of Eq.~\eqref{eq:default-precond} are stated with assumptions in Sections~\ref{sec:theory} and~\ref{sec:method}. Full proofs of the rank-capacity lemma and the block-wise subspace-constraint propositions appear in Appendix~\ref{sec:appendix:proofs}; preconditioner derivations for all four design corners are in Appendix~\ref{sec:appendix:precond}.

    \item {\bf Experimental result reproducibility}
    \item[] Question: Does the paper fully disclose all the information needed to reproduce the main experimental results of the paper to the extent that it affects the main claims and/or conclusions of the paper (regardless of whether the code and data are provided or not)?
    \item[] Answer: \answerYes{}
    \item[] Justification: Section~\ref{sec:method} specifies the \method\ reparameterization (block-wise SVD, core sizes, trainability switches); Section~\ref{sec:exp} (Setup paragraph and per-task subsections) and Appendix~\ref{sec:appendix:exp} list optimizer, learning rates, batch size, schedule, reward definition, rollout protocol, and hardware. Algorithm~\ref{alg:fura} (Appendix~\ref{sec:algorithm}) gives a reference forward-pass implementation.

\item {\bf Open access to data and code}
    \item[] Question: Does the paper provide open access to the data and code, with sufficient instructions to faithfully reproduce the main experimental results, as described in supplemental material?
    \item[] Answer: \answerYes{}
    \item[] Justification: We released our source code \url{https://github.com/yequanzhao/FuRA}. All datasets (Commonsense-170K and MATH-10K via \cite{hu2023llmadapters}, MetaMathQA-100K \cite{yu2023metamath}, \texttt{llava\_v1\_5\_mix665k}) \cite{liu2024dora} and evaluation benchmarks are publicly available and cited in Section~\ref{sec:exp}.

\item {\bf Experimental setting/details}
    \item[] Question: Does the paper specify all the training and test details (e.g., data splits, hyperparameters, how they were chosen, type of optimizer) necessary to understand the results?
    \item[] Answer: \answerYes{}
    \item[] Justification: Section~\ref{sec:exp} (Setup paragraph and the RL/VLM/QFuRA subsections) summarizes tasks, baseline protocols, and the unified \method\ configuration; Appendix~\ref{sec:appendix:exp} details the SFT, GRPO, and visual-instruction-tuning hyperparameters (Tables~\ref{tab:fura-hparams-commonsense}, \ref{tab:fura-hparams-rl}, \ref{tab:fura-hparams-vlm}).

\item {\bf Experiment statistical significance}
    \item[] Question: Does the paper report error bars suitably and correctly defined or other appropriate information about the statistical significance of the experiments?
    \item[] Answer: \answerYes{}
    \item[] Justification: \method\ Commonsense-170K SFT results on LLaMA-3-8B are averaged over 3 seeds with per-seed scores and standard deviations reported in Table~\ref{tab:fura-seed-sweep} (Appendix~\ref{sec:appendix:per_task}); the math RLVR results in Table~\ref{tab:rl_paper} are averaged over 3 independent seeds (Section~\ref{sec:exp:rl}). 

\item {\bf Experiments compute resources}
    \item[] Question: For each experiment, does the paper provide sufficient information on the computer resources (type of compute workers, memory, time of execution) needed to reproduce the experiments?
    \item[] Answer: \answerYes{}
    \item[] Justification: Section~\ref{sec:exp:rl} and Appendix~\ref{sec:appendix:syscost} state that all runs use a single NVIDIA H100 (94GB) GPU. Table~\ref{tab:cost} reports per-step wall-clock and peak memory; Appendix~\ref{sec:appendix:complexity} provides the full per-component breakdown (forward pass, memory, deployment, SVD initialization cost).
    
\item {\bf Code of ethics}
    \item[] Question: Does the research conducted in the paper conform, in every respect, with the NeurIPS Code of Ethics \url{https://neurips.cc/public/EthicsGuidelines}?
    \item[] Answer: \answerYes{}
    \item[] Justification: The work develops a PEFT method evaluated on publicly released benchmarks and pretrained checkpoints; it does not involve human subjects, private data, or deployment decisions. Anonymity and licensing of reused assets are preserved (see items below).

\item {\bf Broader impacts}
    \item[] Question: Does the paper discuss both potential positive societal impacts and negative societal impacts of the work performed?
    \item[] Answer: \answerYes{}
    \item[] Justification: Section~\ref{sec:conclusion} notes the positive impact of broadening access to full-rank, spectrum-aware fine-tuning at LoRA-level cost, and acknowledges that the method inherits the general societal risks of LLM fine-tuning while introducing no new generative capabilities beyond those of the base checkpoints.
    
\item {\bf Safeguards}
    \item[] Question: Does the paper describe safeguards that have been put in place for responsible release of data or models that have a high risk for misuse (e.g., pre-trained language models, image generators, or scraped datasets)?
    \item[] Answer: \answerNA{}
    \item[] Justification: We release an adapter-training method, not new pretrained weights or scraped datasets. The released code operates on existing, publicly distributed checkpoints that already ship with their own safeguards.

\item {\bf Licenses for existing assets}
    \item[] Question: Are the creators or original owners of assets (e.g., code, data, models), used in the paper, properly credited and are the license and terms of use explicitly mentioned and properly respected?
    \item[] Answer: \answerYes{}
    \item[] Justification: All pretrained models (LLaMA-2-7B \citep{touvron2023llama2}, LLaMA-3-8B/70B \citep{grattafiori2024llama3}, Qwen3-1.7B \citep{qwen3}, Qwen2.5-7B \citep{qwen2025qwen25}, LLaVA-1.5-7B \citep{liu2023llava}), datasets (Commonsense-170K and MATH-10K via \citep{hu2023llmadapters}, MetaMathQA \citep{yu2023metamath}, evaluation benchmarks cited individually in Section~\ref{sec:exp}), and tooling (vLLM \citep{kwon2023vllm}, \texttt{bitsandbytes}) are used under their released licenses and cited. The BTT primitive from \citep{qiu2024btt} is reused with citation.

\item {\bf New assets}
    \item[] Question: Are new assets introduced in the paper well documented and is the documentation provided alongside the assets?
    \item[] Answer: \answerYes{}
    \item[] Justification: No new datasets or pretrained weights are released. We provided READMEs, training scripts, and configuration files for each reported run in released code.

\item {\bf Crowdsourcing and research with human subjects}
    \item[] Question: For crowdsourcing experiments and research with human subjects, does the paper include the full text of instructions given to participants and screenshots, if applicable, as well as details about compensation (if any)?
    \item[] Answer: \answerNA{}
    \item[] Justification: The paper does not involve crowdsourcing or human-subject studies; all experiments use publicly available benchmarks.

\item {\bf Institutional review board (IRB) approvals or equivalent for research with human subjects}
    \item[] Question: Does the paper describe potential risks incurred by study participants, whether such risks were disclosed to the subjects, and whether Institutional Review Board (IRB) approvals (or an equivalent approval/review based on the requirements of your country or institution) were obtained?
    \item[] Answer: \answerNA{}
    \item[] Justification: The paper does not involve human subjects; no IRB review is required.

\item {\bf Declaration of LLM usage}
    \item[] Question: Does the paper describe the usage of LLMs if it is an important, original, or non-standard component of the core methods in this research? Note that if the LLM is used only for writing, editing, or formatting purposes and does \emph{not} impact the core methodology, scientific rigor, or originality of the research, declaration is not required.
    \item[] Answer: \answerNA{}
    \item[] Justification: The core method development in this research does not involve LLMs as any important, original, or non-standard components.

\end{enumerate}